\documentclass[suppldata]{interact}
\usepackage{lmodern}
\usepackage[T1]{fontenc}
\usepackage[utf8]{inputenc}
\usepackage{geometry}
\usepackage{color}
\usepackage{babel}
\usepackage{array}
\usepackage{float}
\usepackage{calc}
\usepackage{amsmath}
\usepackage{amsthm}
\usepackage{amssymb}
\usepackage{graphicx}
\usepackage{apacite}
\usepackage[toc]{appendix}
\usepackage{comment}
\usepackage[unicode=true,pdfusetitle,
 bookmarks=true,bookmarksnumbered=false,bookmarksopen=false,
 breaklinks=false,pdfborder={0 0 1},backref=false,colorlinks=false]
 {hyperref}

\usepackage{epstopdf}
\usepackage[caption=false]{subfig}
\usepackage[doublespacing]{setspace}
\usepackage[longnamesfirst,sort]{natbib}
\bibpunct[, ]{(}{)}{;}{a}{,}{,}

\makeatletter

\newcommand{\noun}[1]{\textsc{#1}}

\floatstyle{ruled}
\newfloat{algorithm}{tbp}{loa}
\providecommand{\algorithmname}{Algorithm}
\floatname{algorithm}{\protect\algorithmname}

\theoremstyle{plain}
\newtheorem{thm}{\protect\theoremname}[section]
\theoremstyle{definition}
\newtheorem{defn}[thm]{\protect\definitionname}
\theoremstyle{definition}
\newtheorem{example}[thm]{\protect\examplename}
\theoremstyle{plain}
\newtheorem{assumption}[thm]{\protect\assumptionname}
\theoremstyle{plain}
\newtheorem{prop}[thm]{\protect\propositionname}
\theoremstyle{remark}
\newtheorem{rem}[thm]{\protect\remarkname}
\theoremstyle{remark}
\newtheorem*{acknowledgement*}{\protect\acknowledgementname}
\theoremstyle{plain}
\newtheorem{lem}[thm]{\protect\lemmaname}

\makeatother

\providecommand{\acknowledgementname}{Acknowledgement}
\providecommand{\assumptionname}{Assumption}
\providecommand{\definitionname}{Definition}
\providecommand{\examplename}{Example}
\providecommand{\lemmaname}{Lemma}
\providecommand{\propositionname}{Proposition}
\providecommand{\remarkname}{Remark}
\providecommand{\theoremname}{Theorem}

\begin{document}
	

\title{Online non-convex learning for river pollution source identification}

\author{
\name{Wenjie Huang\textsuperscript{a,b}, 
Jing Jiang\textsuperscript{c}\thanks{} and Xiao Liu\textsuperscript{d,e}\thanks{}}
\affil{\textsuperscript{a}HKU-Musketeers Foundation Institute of Data Science, The University of Hong Kong, Hong Kong SAR; \\ \textsuperscript{b}Department of Industrial and Manufacturing Systems Engineering, The University of Hong Kong, Hong Kong SAR; \\
\textsuperscript{c}JD.com, Inc., Beijing, China;\\
\textsuperscript{d}Department of Industrial Engineering and Management, Shanghai Jiao Tong University, Shanghai, China;\\
\textsuperscript{e}Department of Industrial Systems Engineering and Management, National University of Singapore, Singapore.}
}
\maketitle

\noindent
$*$ Corresponding author: Professor Xiao Liu\\
\#15-02, CREATE Tower, 1 CREATE Way, Singapore 138602\\
Email: x\_liu@sjtu.edu.cn\\
Phone: +65-6601-6129\\
Fax: +65-6601-6129

\newpage
\textbf{\large Online non-convex learning for river pollution source identification}
\vspace{0.5cm}
\maketitle

\begin{abstract}
	In this paper, novel gradient-based online learning algorithms are developed to investigate an important environmental application: real-time river pollution source identification, which aims at estimating the released mass, location, and time of a river pollution source based on downstream sensor data monitoring the pollution concentration. {\color{black}{The pollution is assumed to be instantaneously released once.}} The problem can be formulated as a non-convex loss minimization problem in statistical learning, and our online algorithms have vectorized and adaptive step sizes to ensure high estimation accuracy in three dimensions which have different magnitudes. In order to keep the algorithm from stucking to the saddle points of non-convex loss, the ``escaping from saddle points'' module and multi-start setting are derived to further improve the estimation accuracy by searching for the global minimizer of the loss functions. This can be shown theoretically and experimentally as the $O(N)$ local regret of the algorithms and the high probability cumulative regret bound $O(N)$ under a particular error bound condition in loss functions. A real-life river pollution source identification example shows the superior performance of our algorithms compared to existing methods in terms of estimation accuracy. Managerial insights for the decision maker to use the algorithms are also provided.
	
\end{abstract}


\section{Introduction}
Rivers are one of the most important natural resources; they can be used as sources for water supply, for recreational use, to provide nourishment and a habitat for organisms, and to be viewed as majestic
scenery \citep{sanders_toward_1990}. However, they appear to be vulnerable to water pollution because of
their openness and accessibility to agricultural, industrial,
and municipal processes \citep{thibault_facing_2009}. The release of water pollutants into rivers,
such as from accidental or deliberate sewage leakages from factories or
treatment plants or the intentional release of toxic chemicals, threatens
human health, the living environment, and ecological security. According
to statistics in \cite{ji_accidents_2017}, 373 water pollution
accidents occurred from 2011 to 2015, some of which had serious
consequences. Additionally, the number of water pollution accidents has
been rising \citep{shao_city_2006, qu_multi-stage_2016, guo_mathematical_2019}. To mitigate the negative impacts caused
by such accidents, there is an urgent need to develop an efficient approach
for identifying water pollution sources (i.e., released mass, location, and released time of pollution accidents), which is the foundation for
quick emergency response and timely post-accident remediation.

\subsection{Pollution Source Identification}
The identification of pollution sources in rivers poses a great challenge for two main reasons.
First, some elements in the migration process, such as velocity and dispersivity, that are key factors for the migration process are, time-varying and environmental-condition-based parameters. The problem of the identification of pollution sources, which aims to trace the pollution source based on the observation of the concentration of pollution from downstream sensors, is thus an inverse problem of analyzing the migration process. Such a problem is complex because of the non-linearity and uncertainties of the migration process \citep{wang_new_2018}. 
Second, as streaming observations can be collected, ``online'' decisions on identifying pollution sources need be made to improve the speed of the emergency response \citep{preis_contamination_2006}. 


The existing approaches for identifying pollution sources are mainly categorized into the three streams described below. 
1) \textit{The analytic approach}, which directly solves the ``inverse'' problems (see
\cite{li_global_2011} and \cite{li_estimation_2016}). This uses the downstream monitoring sensor data to identify the pollution source information based on the advection-dispersion equation (ADE) model. Specifically, \cite{li_global_2011}
propose the global multi-quadratic collocation method to identify
multi-point sources in a groundwater system. \cite{li_estimation_2016}
analyze the inverse model for the identification of pollutant sources
in rivers by the global space-time radial basis collocation method, which directly induces the problem to a single-step solution of a system of linear algebraic equations in the entire space-time domain. 2) \textit{The optimization approach} mainly aims to minimize the difference between
the actual observation values and the theoretical values of the pollutant
concentration. The implemented algorithms for analyzing the optimization model are the genetic algorithm \citep{zhang_pollutant_2017},
artificial neural network \citep{srivastava_breakthrough_2014}, simulated
anneal algorithm \citep{jha_three-dimensional_2012}, and differential
evolution algorithm \citep{wang_new_2018}. Both the analytic and optimization approaches are sensitive
to uncertainties and observation noise. 3) \textit{The statistical approach} includes the backward probability method \citep{wang_new_2018} and Bayesian and Markov Chain Monte Carlo
methods \citep{hazart_inverse_2014,yang_multi-point_2016}. \cite{wang_new_2018} obtain the source location probability based on the use of equations governing the pollutant migration process and integrated the linear regression model for the source identification.
\cite{yang_multi-point_2016} convert the problem by computing the posterior probability
distribution of source information. 
The statistical approach has an advantage in dealing with noisy and incomplete prior
information (Hazart et al., 2014).

{\color{black}{These approaches for identifying pollution
	sources fall into the ``offline'' fashion, which means that the identification
	has to be made only after all the monitoring data from sensors are collected. However, in this paper, a different approach, the ``online'' learning algorithm, is developed and analyzed to conduct a real-time estimation of pollution source information, based on streaming sensor data.}}

\subsection{Online Learning}
In a world where automatic data collection is ubiquitous, statisticians
must update their paradigms to cope with new problems. {\color{black}{For an internet network or a financial market,
a common feature emerges: Huge amounts of dynamic data need to
be understood and quickly processed.}} In real-life applications, the monitoring data are not adequate for
accurate estimation at the early implementation stage of the sensors
(i.e., the cold-start problem for recommender systems). Prompt and real-time
action should be made to resolve the economic loss from the pollution as
much as possible once the information of the pollution source has
been estimated. A new approach is required to conduct ``learning''
and ``optimization'' simultaneously. This way,  the real-time pollution sources' identification can be conducted with
streaming data, and the identification result will converge to the true information
as the data size grows. The algorithm will thus have plausible
asymptotic performance. 

Data storage is another reason to limit
offline approaches (mentioned in Section 1.1) for real-life applications. The hardware may
not be able to store the huge amount of monitoring data produced from the sensors,
and the data will become uninformative after the identification has
been made. In terms of the optimization point of view, most of the ADE models (see \cite{van_genuchten_analytical_1982,de_smedt_analytical_2005,wang_new_2018}) involve complex mathematical
structures (e.g. non-convex or non-smooth), such that {\color{black}{the optimality can be attained by heuristic algorithms, but it is difficult to derive the theoretical
guarantee of the performance of the algorithm}}. Explicitly, it will be difficult to figure out the size of the streaming monitoring data such that the identification
error (the gap between the estimated information and the true information
of the pollution source) will be reduced to a certain precision. In this paper, our online pollution source identification techniques are based on online learning and are significantly different from those in current studies. In the
next section, a review of online learning (a.k.a online
optimization) approaches will be given.

An online learning technique
is developed, which consists of a sequence of alternating phases of
observation and optimization, where data are used dynamically to make
decisions in real time. In \cite{shalev-shwartz_online_2012} and \cite{hazan_introduction_2016},
models and algorithms for online optimization are proposed, including: 
online convex optimization (OCO), online classification, online stochastic
optimization, and the limit feedback (Bandit) problem. The goal of the OCO
algorithm is normally to minimize the \emph{cumulative regret} defined
by the following procedures. At each period $t\in\{1,\,2,\,...,\,T\}$,
an online decision maker chooses a feasible strategy ($x_{t}$) from
a decision set ($\mathcal{X}\subset\mathbb{R}^{d}$) and suffers a loss
given by $f_{t}(x_{t})$, where $f_{t}(\cdot)$ is a convex loss function.
One key feature of online optimization is that the decision maker
must make a decision for period $t$ without knowing the loss function
$f_{t}(\cdot)$. The cumulative regret compares the accumulated loss
suffered by the player with the loss suffered by the best fixed strategy.
Specifically, the cumulative regret for algorithm $\mathcal{A}$ is
defined as $\mathfrak{R}{}_{T}^{\mathcal{A}}(\{x_{t}\}_{1}^{T})=\sum_{t=1}^{T}f_{t}(x_{t})-\min_{x\in\mathcal{X}}\sum_{t=1}^{T}f_{t}(x)$. In \cite{hazan_introduction_2016}, it is shown that the lower and upper regret bounds
of OCO are to be $\Omega(\sqrt{T})$ and $O(\sqrt{T})$, respectively.
If the loss functions are strongly convex, logarithmic bounds on the
regret, that is, $O(\log T)$, can be established. As suggested by its name,
the loss functions in the OCO are assumed to be ``convex.'' Only a handful of
papers study online learning with non-convex loss functions. In
\cite{gao_online_2018}, a non-stationary regret is proposed as a performance
metric, a gradient-based algorithm is proposed for online non-convex
learning, and the complexity bound is also derived $O(\sqrt{T})$ under the
pseudo-convexity condition. In \cite{hazan_efficient_2017}, a local regret measure
is proposed, and the regret bound of the gradient-based algorithm is derived
under mild conditions on loss functions. In \cite{maillard_online_2010},
they consider the problem of online learning in an adversarial environment
when the reward functions chosen by the adversary are assumed to be
Lipschitz. The cumulative regret is upper bounded by $O(\sqrt{T})$
under special geometric considerations. In \cite{suggala_online_2019} and \cite{agarwal_generalization_2009},
Follow-the-Leader algorithms are proposed, which can convert online
learning into an offline optimization oracle. By slightly strengthening
the oracle model, they make the online learning and offline models computationally
equivalent. In \cite{zhang_online_2015}, the bandit algorithm for non-convex
learning is developed where the loss function specifically works on
the domain of the products of the decision maker's action and the adversary's action.

Other than standard gradient-based methods, in \cite{yang_recursive_2017, yang_optimal_2018},
online exponential weighting algorithms are developed for the online non-convex
learning problem even though the decision set is non-convex. The regret bound
$O(\sqrt{T\log(T)})$ is the best of our knowledge in the literature.
In \cite{ge_escaping_2015}, \cite{jin_how_2017}, and \cite{levin_choice_2006}, the authors mention 
that even for the offline non-convex optimization problem, the stationary
points of the non-convex objective may be the local minimum, local maximum,
and the saddle point. They design algorithms that enable escaping
from the saddle point and reaching the local minimum with high probability.
The main idea is to use random perturbed gradient and check the second-order (eigenvalue of Hessian matrix) condition.

\subsection{Contributions and Organization}

In this paper, online learning
approaches specifically for real-time river pollution source identification are developed and analyzed. {\color{black}{The main contributions are as follows: }}
{\color{black}{
\begin{enumerate}
	\item The properties of the specific ADE model studied in \cite{wang_new_2018} are analyzed. Our loss
	function satisfies certain properties to quantify the estimation
	error between monitoring data and the output of the ADE model, analogous to
	the objective of statistical learning. It can be found that the loss function also exhibits Lipschitz continuity for its first and second derivatives.
	This observation inspires us to develop a gradient-based online algorithm
	to conduct pollution source identification. Based on \cite{hazan_efficient_2017}, new online non-convex learning algorithms are developed with modified
	step sizes. The step sizes are set to be vectors, meaning that the step sizes
	on different dimensions (released mass, location, and released time) are
	different. This setting avoids the algorithm's slow update on some
	dimensions and its divergence on others, which improves the robustness
	of normal gradient-descent-type algorithms on the pollution source
	identification problems. Moreover, our algorithm is equipped with the backtracking line-search-type technique to improve the convergence performance of the algorithm
	to the stationary points. Our algorithm also enables escaping from saddle points and finding a local minimum with a high
	probability, which will improve the quality of its real-time pollution source estimation.
	\item The performance analysis of our online
	learning algorithms is accomplished by deriving the explicit $O(N/w^{2}+N/w+N)$ ``local''
	regret bound ($N$
	is the size of the monitoring data, and $w$ is the window size), and
	also deriving the total number of gradient estimation required as 
	$O(N/w)$ when the algorithm terminates. It can be found that the theoretical guarantee of the algorithms
	is independent of the line search method applied to adjust the step sizes;
	thus, an analytic framework for gradient-based online
	non-convex optimization algorithms with adjustable step-sizes is developed. The optimal
	window size can be shown such that the local regret bound can be minimized under
	fixed computational resources, that is, a fixed number of gradient estimations.
	The optimal window size (i.e., the size of historical data storage)
	is dependent on the specific properties of the pollution source (choice
	of the ADE model), the parameters of the river (choice of the ADE model
	parameters), and the accuracy tolerance of the algorithm. {\color{black}{A mechanism is developed to determine the minimal number of sensors such that the measurement error can be controlled.}} 
	For the extended algorithm with the ``escaping from saddle points'' module, its local regret bound has the same complexity as the basic algorithm. The total number of gradient estimations of the extended algorithm to find a local minimum
	with a high probability is also shown. When the losses follow a specific error
	bound condition, the local regret bound can be linked with cumulative regret, which has been widely used as a metric for OCO with the global optimum guarantee when the loss function
	and ADE model are specifically selected. 
\end{enumerate}
}}

Our algorithms are implemented on Rhodamine WT dye concentration data from a travel time study on
the Truckee River between Glenshire Drive near Truckee, California and
Mogul, Nevada \cite{crompton_traveltime_2008}. The experiment results validate all our theoretical regret bounds and show that our algorithms are superior to existing online algorithms on all dimensions (e.g., released mass, location, and time). This shows that the multi-start module and ``escaping from saddle point'' module can achieve a significantly low estimation error on certain dimensions. 

The rest of the paper is organized as follows. Section 2 introduces the key notations used throughout the paper
and problem settings for pollution source identification. \textcolor{black}{The problem is formulated, and its properties are derived}. Section 3 contains
the development of our online learning algorithm, the performance
metrics (i.e., local and cumulative regret) of the algorithm, and various performance
analyses. Section 4 serves as a remark on the issue of the ``escaping from saddle points'' module, which includes a modified algorithm that enables
finding the local minimum with a high probability. In Section 5, our algorithms are applied to a real-life river
pollution source identification example, and we describe how we experimentally test the regret
bounds and compare the performance of variants of online algorithms and existing methods in literature. The paper
concludes in Section 6 with future research directions. 

\section{Preliminaries}

In this section, the key notations throughout the paper
and problem settings for pollution source identification are developed, along with the dispersion properties and model of the pollution sources.

\subsection{Notations }

For vectors, $\|\cdot\|$ denotes the $l_{2}$-norm,
$\|\cdot\|_{\infty}$ denotes the infinity norm, and $\|\cdot\|_{\min}$
outputs the minimum element of the vector. The symbol $\otimes$ denotes the Hadamard
product of vectors, and $\oslash$ denotes the element-wise division of between two vectors. For matrices, $\lambda_{\textrm{min}}(\cdot)$ denotes the smallest eigenvalues.
For a function $f:\,\mathbb{R}^{d}\rightarrow\mathbb{R}$, $\nabla_{x}f(\cdot)$ and $\nabla_{x}^{2}f(\cdot)$ denote the
gradient and Hessian with respect to $x$, and we use $\partial f(\cdot)/\partial x$
to denote the partial derivative on $x$. We use $\lceil\cdot\rceil$ to denote rounding up the value to an integer, and the computational complexity
notation $O(\cdot)$ to hide only absolute constants that do not
depend on any problem parameter. Let $\mathbb{B}_{0}(r)$ denote
the $d$-dimensional ball centered at the origin with the radius $r$. We use $\Pi_{\mathcal{F}}$ to denote the projection onto the set $\mathcal{F}$
defined in a Euclidean distance sense. 

\subsection{Problem Settings } \label{problem settings}

Figure 1 shows the geometric illustration of our problem. {\color{black}{The figure presents a part of the Yangtze river in Jiangsu Province, China}}. There are facilities including
factories and hospitals (marked by blue triangles) located alongside
the river. These facilities are the potential sources of discharged
pollutants. The set of sensors (denoted by $M$) placed in downstream
cross sections are designed to monitor the $\emph{concentration}$
of the pollutant source, which is expressed in terms of mass per-unit
volume. A ``sampling'' is defined to be a collection of pollutant concentrations
detected by a sensor {\color{black}{$m\in M$}} at one time. Let $N$
denote the number of samplings. For each sensor, the set of samplings detected by it is denoted $N_m$. For the $n$th ($n\in \{1,...,N\}$) sampling,
the concentration detected by sensor $m$ is denoted by $c_{m}^{n}$,
and the time for the concentration data collected by sensor $m$ is
denoted by $t_{m}^{n}$. Let $l_{m},\,m\in M$ denote the location
of the sensor $m$ and $L:=\{l_{m}\}_{m\in M}$ represent the set
of the location of all the sensors. Let $C^{n}:=\{c_{m}^{n}\}_{m\in M}$
denote the set of concentrations collected by all sensors for the $n$th
sampling. Once water pollution accident occurs, sensors $m\in M$
can detect a series of pollutant concentrations ({\color{black}{$c_{m}^{n}$}}) over the dynamic
sampling process. Throughout the paper, assume that there is only one pollution source, and it is instantaneously released once.

{\color{black}{The identification problem studied in this paper has an instantaneously released pollution source with initial value.}} The objective of this research is to develop online
learning algorithms, using the collected data ($\{c_{m}^{n}\}_{m\in M}$) for each sampling $n\in N$ to estimate the information of the true pollution source in real time, including the released mass ($s$), the source location ($l$), and
the released time ($t$). Throughout the paper, let $(s^{n}_{m},\,l^{n}_{m},\,t^{n}_{m})$
denote the estimated pollution source information at period $n$ by sensor $m$ and $(s^{\ast},\,l^{\ast},\,t^{\ast})$
denote the true pollution source information.

\begin{figure}[H]
\begin{centering}
\includegraphics[scale=.3]{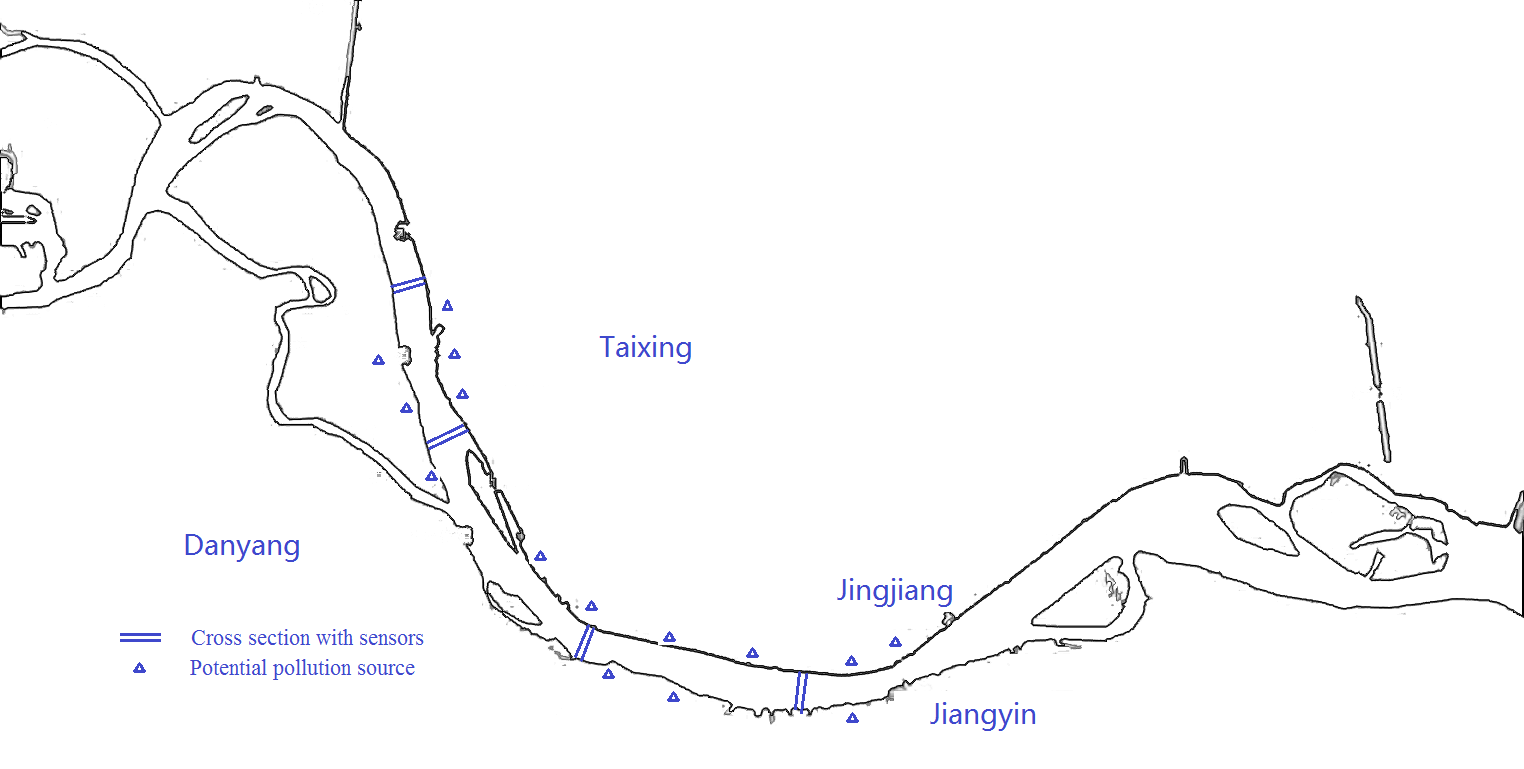}
\par\end{centering}
\caption{The river layout}
\end{figure}

\subsection{Advection-Dispersion Equation}\label{ADE}

In this paper, the water pollutants are assumed to be released instantaneously from pollution sources. Once pollutants with the initial concentration
$s$ are discharged in location $l$ and at time $t$, pollutants migrate
and diffuse along the direction of water flow, and the pollutant concentration
varies at different downstream cross sections considering the hydrodynamic characteristics.
ADE (see \cite{wang_new_2018}) is
commonly used to explain how the pollutants migrate and diffuse in
the river. The analytical expression of ADE is defined as
{\color{black}{\begin{equation}
	C(l_{m},\,t_{m}^{n}|s,\,l,\,t):=\frac{s}{A\sqrt{4\pi D\left(t_{m}^{n}-t\right)}}\exp\left[-\frac{\left(l_{m}-l-v\left(t_{m}^{n}-t\right)\right)^{2}}{4D\left(t_{m}^{n}-t\right)}\right]\exp\left(-k\left(t_{m}^{n}-t\right)\right),\label{IM-1}
\end{equation}}}where $A$ is the water area perpendicular to the river flow direction
($\textrm{meter}^{2}$); $D$ is the dispersion coefficient in the flow direction,
which is generally evaluated by empirical equations or based on experiments
($\textrm{meter}^{2}/\textrm{second}$); $v$ is the mean flow velocity of the cross section
($\textrm{meter}/\textrm{second}$); and $k$ is the decay coefficient ($\textrm{second}^{-1}$). {\color{black}{In Equation (1), the term $l_{m}-l$ computes the length of $m$ cross section to the pollution source, along the river (geometrically, it is a curve). 

In this case, our method is general and not limited to a one-dimensional problem.}} The physical
meaning of Equation (\ref{IM-1}) is explained as follows: Given
the known pollution source information $(s,\,l,\,t)$, the theoretical estimation of
the concentration level monitored by sensor $m$ (at location $l_{m}$)
at time $t_{m}^{n}$ equals {\color{black}{$C(l_{m},\,t_{m}^{n}|s,\,l,\,t)$}}.

{\color{black}{The ``reverse'' formulation of Equation (1) is denoted as $C(s,\,l,\,t|l_{m},\,t_{m}^{n})$, where the parameters and variables are exchanged. For this formulation, given sensor $m$ (at location $l_{m}$)
at time $t_{m}^{n}$, the downstreaming concentration level can be predicted, given the estimated $(s,\,l,\,t)$. }}

\subsection{Problem Formulation}
\textcolor{black}{As described in Section \ref{problem settings}, the goal of the proposed method in this paper is to estimate the source information $(s,\,l,\,t)$, given the
streaming concentration data collected by each sensor $c_{m}^{n}$. Equation (\ref{IM}) is then used to measure the gap between the theoretical
estimation $C(s,\,l,\,t|l_{m},t_{m}^{n})$ value and the real data $c_{m}^{n}$: }
\begin{align}
\tilde{C}(s,\,l,\,t|l_{m},\,t_{m}^{n},\,c_{m}^{n}):=&\frac{s}{A\sqrt{4\pi D\left(t_{m}^{n}-t\right)}}\exp\left[-\frac{\left(l_{m}-l-v\left(t_{m}^{n}-t\right)\right)^{2}}{4D\left(t_{m}^{n}-t\right)}\right] \nonumber
\\
&\times\exp\left(-k\left(t_{m}^{n}-t\right)\right)-c_{m}^{n}.\label{IM}
\end{align}
For simplicity, use $\tilde{C}_{n,m}(s,\,l,\,t)$ to represent $\tilde{C}(s,\,l,\,t|l_{m},\,t_{m}^{n},\,c_{m}^{n})$
as defined by Eq. (\ref{IM}) throughout the paper. Instead of providing a theoretical estimation
of the downstream concentration level, given the known information
$(s,\,l,\,t)$, the ADE model is applied in a reverse-engineered manner, using $\tilde{C}_{n,m}(s,\,l,\,t)$ to indicate the estimation error
of an estimated information $(s,\,l,\,t)$. Given any pollution source
information estimation $(s,\,l,\,t)$, suppose $\tilde{C}_{n,m}(s,\,l,\,t)>0$.
Then, the estimated concentration monitored by sensor $m$ at period
$n$ is \emph{overestimated}, and $\tilde{C}_{n,m}(s,\,l,\,t)<0$ means
the estimated concentration is \emph{underestimated}. 

{\color{black}{Given the estimated information $(s,\,l,\,t)$, the square function is used to quantify the estimation error of sensor $m$ at period $n$, denoted as $\left(\tilde{C}_{n,m}(s,\,l,\,t)\right)^{2}$. Namely, the estimation error of the $n$th period over sensor $m$ is defined as 
\begin{align}
	\Psi(s,\,l,\,t|l_{m},\,t_{m}^{n},\,c_{m}^{n}) & =\left(\tilde{C}_{n,m}(s,\,l,\,t)\right)^{2}.\label{Loss}
\end{align}}}Use $\Psi^{n}_{m}(s,l,t)$ to represent $\Psi(s,\,l,\,t|l_{m},\,t_{m}^{n},\,c_{m}^{n})$
for simplicity. {\color{black}{The function $\Psi^{n}_{m}(s,\,l,\,t)$ on $(s,\,l,\,t)$ is called the \emph{loss function}.  In statistical learning, when the data $t^{n}_{m},\,c^{n}_{m}$ for all $n\in N$ are known in advance, the optimal source information can be identified by minimizing the average loss function $\sum_{n\in N}\Psi_{m}^{n}(s,\,l,\,t)$. In contrast, in online learning cases, the data $t^{n}_{m},\,c^{n}_{m}$ as well as the corresponding loss function $\Psi_{m}^{n}(s,\,l,\,t)$ are not completely known in advance but are sequentially accessible to the decision maker. Thus, \emph{regret} the minimization method will be served as a new performance evaluation criterion other than the average loss minimization such that the estimation error of $(s,\,l,\,t)$ can be measured dynamically once the new sampling of concentration data is incorporated. In this paper, two ``regret'' definitions are presented, which measure
the performance of online learning algorithms.}} They are called ``cumulative regret'' and ``local regret.'' In the well-established framework of OCO, numerous algorithms can efficiently achieve the optimal cumulative 
regret in the sense of converging in the total accumulated loss
toward the best fixed decision in hindsight (see \cite{hazan_introduction_2016} and \cite{shalev-shwartz_online_2012}). The ``cumulative regret'' is defined as follows.
{\color{black}{\begin{defn}
	(Cumulative regret) Given an online learning algorithm
	$\mathcal{A}$, its cumulative regret with respect to sensor $m$ after
	$N$ iterations is denoted by
	\begin{equation}
		\mathfrak{R}_{m}^{\mathcal{A}}(N):=\sum_{n=1}^{N}\Psi^{n}_{m}(x^{n}_{m})-\inf_{x\in\mathcal{F}}\sum_{n=1}^{N}\Psi^{n}_{m}(x).  \label{Cumulative}
	\end{equation}
\end{defn}
In the well-established framework of OCO, numerous algorithms can efficiently achieve optimal regret in the sense of converging in the average
loss toward the best fixed decision in hindsight. That is, one can iterate $x_{m}^{1},\,...,\,x_{m}^{N}$ such that the \emph{long run average cumulative regret} of $\mathcal{A}$,
\[\mathfrak{R}_{m}^{\mathcal{A}}(N)/N = \frac{1}{N}\sum_{n=1}^{N}\left[\Psi^{n}_{m}(x_{m}^{n})-\Psi^{n}_{m}(x)\right]=o(1).\]
}}
However, even in the offline non-convex optimization case,
it is too ambitious to converge toward a global minimizer in hindsight.
The global convergence of offline non-convex optimization is normally NP-hard {\color{black}{(see \cite{hazan_efficient_2017})}}. In the
existing literature, it is common to state convergence guarantees toward
an $\epsilon$-approximate stationary point; that is, there exists
some iteration $x^{n}_{m}$ for which $\|\nabla\Psi^{n}_{m}(x^{n}_{m})\|\leq\epsilon$.
Given the computational intractability of direct analogues of
convex regret, the definition of ``local regret'' is given; it is a new notion of regret that
quantifies the objective of predicting points with small gradients
on average.

Throughout this paper, for convenience, the following
notation from \cite{hazan_efficient_2017} is used to denote the sliding-window
time average of functions, parameterized by some window size $1\leq w\leq N$:
\[
F^{n}_{m,\,w}(x):=\frac{1}{w}\sum_{i=0}^{w-1}\Psi^{n-i}_{m}(x). 
\]
For simplicity, define $\Psi^{n}_{m}(x)$ to be identically
zero for all $n\leq0$.
\begin{defn}
(Local regret) \cite[Definition 2.5]{hazan_efficient_2017} Fix some $\eta>0$. Define the $w$-local regret of
an online algorithm with respect to sensor $m$ as
\begin{equation}
\mathfrak{R}_{m, w}(N):=\sum_{n=1}^{N}\left\Vert \nabla_{\mathcal{F},\eta}F^{n}_{m,w}(x^{n}_{m})\right\Vert ^{2}.\label{Local regret}
\end{equation}
\end{defn}
In \cite{hazan_efficient_2017}, the necessity of the time-smoothing term $w$ is
argued such that for any online algorithm, an adversarial sequence of loss
functions can force the local regret incurred to scale with $N$ as
$\Omega(N/w^{2})$. {\color{black}{This result indicates that ``acceptable'' online non-convex algorithms will only be able to achieve the linear local regret $O(N)$. 
}} 
In addition, the
cumulative regret of online non-convex learning algorithms can be measured if the local properties on the local minimizer
of $\Psi^{n}_{m}$ and $\sum_{n=1}^{N}\Psi^{n}_{m}$ are attained. 

Furthermore, for some online learning problem where
$F^{n}_{m,\,w}(x)\approx\Psi^{n}_{m}(x)$, a bound on the local regret truly captures
a guarantee of playing points with small gradients. {\color{black}{However, for any algorithm treating $F^{n}_{m,w}$ as the follow-the-leader objective when
estimating $x^{n+1}_{m}$, the objective considers recording and
summing up $w$ previous loss functions rather than all the historical
loss functions up to iteration $n$. This setting saves the storage
when the online algorithm is implemented.}}

\subsection{Properties of Loss Functions}
\textcolor{black}{In this section, the
mathematical properties of functions $\tilde{C}_{n,m}$ and $\Psi_{m}^{n}$ are figured, which underlines the development of the online algorithms.} 
Throughout this paper, we denote the decision variables $x=(s,\,l,\,t)$ and $x^{\prime}=(s^{\prime},\,l^{\prime},\,t^{\prime})$
for simplicity. We use $\mathcal{S}$, $\mathcal{L}$ and $\mathcal{T}$ to denote the closed feasible set of $(s,\,l,\,t)$ and $\mathcal{F}:=\mathcal{S}\times\mathcal{L}\times\mathcal{T}$. The following assumptions on the
set $\mathcal{F}$ are given first. 
\begin{assumption}
\label{Assumption_1} (i) $\mathcal{F}$ is a convex and compact set.
(ii) For any $t\in\mathcal{T}$, $\min_{n\in N,m\in M}\left\{ t_{m}^{n}\right\} >t$
holds. (iii) There exists $B>0$ such that $|\Psi_{m}^{n}(x)|\leq B$
for any $x\in\mathcal{F}$;
\end{assumption}

Assumption \ref{Assumption_1} (ii) indicates that any sensor monitors data collected after the pollution is discharged. {\color{black}{Based on Assumption
\ref{Assumption_1}, the following Proposition \ref{prop_C_lipschitz_continuous} shows the ``Lipschitz continuity'' property of  $\tilde{C}_{n,m}$ and its gradient on $\mathcal{F}$ by taking the first derivatives. Here, ``Lipschitz continuity'' is first defined  in {\color{black}{Definition 2.4}} as follows.
\begin{defn}
	A function $f$ from $\mathcal{D} \subset \mathbb{R}^{d}$ to $\mathbb{R}$ is ``Lipschitz continuous'' at $x\in \mathcal{D}$ if there is a constant $K$ such that $|f(y)-f(x)|\leq K\|y-x\|$ for all $y\in\mathcal{D}$ sufficiently near $x$. 
\end{defn}

\begin{prop}
	\label{prop_C_lipschitz_continuous} (i) $\tilde{C}_{n,m}$ is bounded
	and Lipschitz continuous on the set $\mathcal{F}$: Given any $x,\,x^{\prime}\in\mathcal{F}$,
	there exists $\sigma>0$ such that $\left|\tilde{C}_{n,m}(x)-\tilde{C}_{n,m}(x^{\prime})\right|\leq\sigma\|x-x^{\prime}\|$.
	(ii) The gradient of $\tilde{C}_{n,m}(x)$ is bounded and
	Lipschitz continuous on the set $\mathcal{F}$: Given any $x,\,x^{\prime}\in\mathcal{F}$,
	there exists $\gamma>0$ such that $\left\Vert \nabla_{x}\tilde{C}_{n,m}(x)-\nabla_{x}\tilde{C}_{n,m}(x^{\prime})\right\Vert \leq \gamma\|x-x^{\prime}\|$.
\end{prop}}}
The following proposition shows the Lipschitz continuity properties
of function $\Psi_{m}^{n}$.
\begin{prop}
\label{Proposition 2.5} (i) $\Psi_{m}^{n}$ is Lipschitz continuous on $\mathcal{F}$: Given any $x,\,x^{\prime}\in\mathcal{F}$, there exists
$\kappa>0$ such that $\left\Vert \Psi_{m}^{n}(x)-\Psi_{m}^{n}(x^{\prime})\right\Vert \leq\kappa\|x-x^{\prime}\|$.
(ii) $\Psi_{m}^{n}$ is $\beta$-smooth on $\mathcal{F}$: Given any $x,\,x^{\prime}\in\mathcal{F}$,
there exists $\beta>0$ such that $\left\Vert \nabla_{x}\Psi_{m}^{n}(x)-\nabla_{x}\Psi_{m}^{n}(x^{\prime})\right\Vert \leq\beta\left\Vert x-x^{\prime}\right\Vert$.
(iii) $\Psi_{m}^{n}$ is also $\iota$-Hessian Lipschitz: Given any $x,\,x^{\prime}\in\mathcal{F}$,
there exists $\iota>0$ such that $\|\nabla_{(s,\,l,\,t)}^{2}\Psi_{m}^{n}(x)-\nabla_{(s,\,l,\,t)}^{2}\Psi_{m}^{n}(x^{\prime})\|\leq \iota\left\Vert x-x^{\prime}\right\Vert$.
\end{prop}

Next, the definition of ``projected gradient'' is provided, and its properties are derived. 
These properties will be useful in designing and analyzing our online
algorithms. The properties are naturally extended from \cite{hazan_efficient_2017},
where the step size $\eta$ is univariate rather than a vector, as in our
case.
\begin{defn}
\label{Definition 2.4} (Projected gradient) Let $\Psi_{m}^{n}:\,\mathcal{F}\rightarrow\mathbb{R}$
be a differentiable function on the compact and 
convex set $\mathcal{F}$. Let vector-valued step size be $\eta\in\mathbb{R}_{+}^{d}$
(throughout the paper, $d=3$ for pollution source identification problems), and we define $\nabla_{\mathcal{F},\eta}\Psi_{m}^{n}:\,\mathcal{F}\rightarrow\mathbb{R}^{d}$,
the $(\mathcal{F},\,\eta)$-projected gradient of $\Psi_{m}^{n}$ by $\nabla_{\mathcal{F},\eta}\Psi_{m}^{n}(x):=\left(x-\Pi_{\mathcal{F}}\left[x-\eta\otimes\nabla_{x}\Psi_{m}^{n}(x)\right]\right)\oslash\eta$,
where $\Pi_{\mathcal{F}}[\cdot]$ denotes the orthogonal projection
onto $\mathcal{F}$, the symbol $\otimes$ denotes the Hadamard
product, and $\oslash$ denotes the element-wise division of two vectors. 
\end{defn}

The following proposition shows that there always exists a point with a vanishing
projected gradient.

\begin{prop}
\label{Proposition 2.7} Let $\mathcal{F}$ be a compact and convex set
and suppose $\Psi_{m}^{n}:\mathcal{F}\rightarrow\mathbb{R}$ satisfies
the properties in Proposition \ref{Proposition 2.5}. Then, there
exists some point $x^{\ast}\in\mathcal{F}$ for which $\nabla_{\mathcal{F},\eta}\Psi_{m}^{n}(x^{\ast})=0$.
\end{prop}

The following proposition states that an approximate
local minimum, as measured by a small projected gradient, is robust
with respect to small perturbations. {\color{black}{This proposition will be applied to derive the local regret bound of our proposed algorithms in Section 3.}}
\begin{prop}
\label{Proposition 2.6} Let $x$ be any point in $\mathcal{F},$ and
let $\Psi$ and $\Phi$ be differentiable functions $\mathcal{F}\rightarrow\mathbb{R}$.
Then, for any $\eta\in\mathbb{R}_{+}^{d}$ and $x\in\mathcal{F}$, $\left\Vert \nabla_{\mathcal{F},\,\eta}[\Psi+\Phi](x)\right\Vert \leq\left\Vert \nabla_{\mathcal{F},\,\eta}\Psi(x)\right\Vert +\left\Vert \nabla\Phi(x)\right\Vert$.
\end{prop}

{\color{black}{Note that the proofs of Proposition \ref{prop_C_lipschitz_continuous}, \ref{Proposition 2.5}, \ref{Proposition 2.7},  and \ref{Proposition 2.6} are all provided in Appendix B. }}
\section{The Algorithm}
In this section, our basic algorithm (Algorithm
1), adaptive time-smoothed online gradient descent (ATGD), is developed. The key
idea of this algorithm is to implement follow-the-leader iterations approximated
to a suitable tolerance using projected gradient descent. Apart from the standard gradient descent paradigm, the improvements of our algorithm are: (i)
adjusting the step sizes to be vectors, meaning that the algorithm
performs different step sizes on different dimensions (released mass, location and released time). This setting avoids potentially slow updates on some dimensions
and divergence on other dimensions, owing to the different magnitude
of the gradient on different dimensions (i.e., a steep surface will have high
gradient magnitude, and a flat surface will have low gradient magnitude). The vectorized step-size
setting thus improves the robustness of the algorithm more than the univariant
step-size version; (ii) the step-size is improved by line search and leads to faster convergence near the stationary points. 

In Algorithm 1, $\mathfrak{S}(F^{n}_{m,w},\,K)$ is an operator choosing the initial step sizes, whose explicit definition will be given later. The explicit description of the normalized Backtracking-Armijo
line search is shown in Algorithm 2, where a properly chosen $\beta$ prevents the step size
from getting too large along the gradient descent direction, and the
backtracking method prevents the step size from getting too small. Algorithm 2 is equipped with adaptive
step sizes modified from the Backtracking-Armijo line search, which is one of the most widely investigated exact line search
methods (see \cite{dennis_jr_numerical_1996,armijo_minimization_1966,nocedal_numerical_2006}). 
Our new Armijo condition is imposed on the normalized gradient to further
reduce the negative effect of a gradient magnitude difference on each dimension, in addition to the vector step-size setting. Term $\eta^{(l)}\beta\|\nabla_{\mathcal{F},\eta^{(l)}}F^{n}_{m,\,w}(x_{m}^{n+1})\|_{2}$,
at the right side of the Armijo inequality, will approach zero when
the norm of the projected gradient approaches zero, and then the condition
will become the standard Backtracking line search. This setting leads to faster convergence near the stationary points. {\color{black}{The following theorem indicates that Algorithm 2 terminates in a finite number of iterations. Its proof is given in Appendix C. }}

{\color{black}{\begin{thm} \label{Theorem 3.1}
	Algorithm 2 terminates in a finite number of iterations. 
\end{thm}}

\begin{algorithm}
\caption{ATGD}

\textbf{Input}: sensor $m\in M$; window size $w\geq1$, tolerance $\delta>0$, constant $K>0$, and a convex set $\mathcal{F}$;

\textbf{Set} $x^{1}_{m}\in\mathcal{F}$ arbitrarily

\textbf{for} $n=1,\,...,\,N$ \textbf{do}

$\quad$Observe the cost function $\Psi^{n}_{m}:\,\mathcal{F}\rightarrow\mathbb{R}$.

$\quad$Compute the initial step-size: $\eta^{0} = \mathfrak{S}(F^{n}_{m,w},\,K)$. 

$\quad$Initialize $x^{n+1}_{m}:=x^{n}_{m}$ and  $\eta^{n}:=\eta^{0}$.

$\quad\quad$ \textbf{while} $\|\nabla_{\mathcal{F},\eta^{n}}F^{n}_{m,\,w}(x^{n+1}_{m})\|_{2}>\delta$
\textbf{do}

$\quad\quad\quad$Determine $\eta^{n}$ using Normalized Backtracking-Armijo line
search (Algorithm 2).

$\quad\quad\quad$Update $x^{n+1}_{m}:=x^{n+1}_{m}-\eta^{n}\otimes\nabla_{\mathcal{F},\eta^{n}}F^{n}_{m,\,w}(x^{n+1}_{m})$.

$\quad\quad$\textbf{end while}

\textbf{end for}
\end{algorithm}

\begin{algorithm}
\caption{Normalized Backtracking-Armijo line search}

\textbf{Input}: $\eta^{0}>0$, $x^{n+1}_{m}$ and $F^{n}_{m,\,w}$, let $\eta^{(0)}=\eta^{0}$
and $l=0$.

\textbf{while $F^{n}_{m,\,w}(x^{n+1}_{m}-\eta^{(l)}\otimes\nabla_{\mathcal{F},\eta^{(l)}}F^{n}_{m,\,w}(x^{n+1}_{m})/\|\nabla_{\mathcal{F},\eta^{(l)}}F^{n}_{m,\,w}(x^{n+1}_{m})\|_{2})>F^{n}_{m,\,w}(x^{n+1}_{m})+\beta\|\eta^{(l)}\otimes\nabla_{\mathcal{F},\eta^{(l)}}F^{n,\,w}_{m}(x^{n+1}_{m})\|_{2}$ do}

$\quad$Set $\eta^{(l+1)}=\tau\otimes\eta^{(l)}$, where $\tau\in(0,\,1)^{d}$
is fixed (e.g., $\tau=(1/2,\,1/2,\,1/2)$),

$\quad$Set $l=l+1$

\textbf{end while}

Set $\eta^{n}=\eta^{(l)}$
\end{algorithm}

In terms of how proper initial step sizes can be chosen, the main idea is to set the initial step sizes on each dimension inversely proportional to the Lipschitz modulus of the loss on that dimension such that the algorithm can adjust the identification change on each dimension to a similar magnitude. A simple partition method is first derived to derive a lower bound on the Lipschitz modulus on each dimension from {\cite{gimbutas_global_2016,sergeyev_introduction_2013,pinter_global_2013}. {\color{black}{The initial step-size selection is presented by Algorithm 3. For simplicity, use the operator $\mathfrak{S}$ to encode the process of finding an initial step size in Algorithm 3. Given the function $F^{n}_{m,w}$ and constant $K>0$, the initial step size is $\eta^{0} = \mathfrak{S}(F^{n}_{m,w},\,K)$. 
		
		\begin{algorithm}
			\caption{Initial step-sizes construction}
			
			\textbf{Input}: Function $F^{n}_{m,\,w}$ and $K>0$. Let $\eta^{0}\in\mathbb{R}^{d}_{+}$ denote initial step sizes at iteration $n$ with the step sizes on each dimension as $\eta_{s}^{0}$, $\eta_{l}^{0}$ and $\eta_{t}^{0}$. 
			
			
			$\quad$ Generate $K^{3}$ points on cube $\mathcal{F}$, denoted as  $(s^{1},...,\,s^{K})$, $(l^{1},...,\,l^{K})$ and $(t^{1},...,\,t^{K})$. In each dimension, {\color{black}{the points are uniformly spaced in cube $\mathcal{F}$.}}
			
			$\quad$ Compute
			\[
			\kappa^{n}_{s} = \max_{i,j,k,p = 1,...,K} \left\{\frac{|F^{n}_{m,w}(s^{i},\,l^k,t^{p})-F^{n}_{m,w}(s^{j},\,l^k,t^{p})|}{|s^{i}-s^{j}|}:\,i\neq j\right\},
			\]
			\[
			\kappa^{n}_{l} = \max_{i,j,k,p = 1,...,K} \left\{\frac{|F^{n}_{m,w}(s^{k},\,l^i,t^{p})-F^{n}_{m,w}(s^{k},\,l^j,t^{p})|}{|l^{i}-l^{j}|}:\,i\neq j\right\},
			\]
			and
			\[
			\kappa^{n}_{t} = \max_{i,j,k,p = 1,...,K} \left\{\frac{|F^{n}_{m,w}(s^{k},\,l^p,t^{i})-F^{n}_{m,w}(s^{k},\,l^p,t^{j})|}{|t^{i}-t^{j}|}:\,i\neq j\right\},
			\]
			to be the Lipschitz modulus estimation on dimension $s$, $l$, and $t$ respectively.
			
			$\quad$ Update $\eta^{0}$ such that it satisfies the following two conditions: (i) for the validation and analysis of the algorithm, choose a sufficiently large initial step size such that $\|\eta^{0}\|_{\min}/\|\eta^{0}\|_{\infty}^{2}\gg \beta/2$, and (ii) the equations 
			$\eta^{0}_{s}/\kappa^{n}_{s} = \eta^{0}_{l}/\kappa^{n}_{l} = \eta^{0}_{t}/\kappa^{n}_{t}$ hold.
			
			
		\end{algorithm}
}}

\subsection{Performance Analysis}

In this subsection, the analysis of the regret bound {\color{black}{(i.e., the upper bound on the local regret)}} for Algorithm 1 is shown
in addition to the derivation of the optimal window size to
minimize the regret bound within a limited number of gradient estimations.

The following Theorem \ref{Theorem 4.1} presents the local regret
and the total number of required gradient estimations. The proof of
Theorem \ref{Theorem 4.1} is given in Appendix C.
{\color{black}{\begin{thm}
			\label{Theorem 4.1} Let $N$ denote the number of samplings. Given sensor $m\in M$, let $\Psi^{1}_{m},\,...,\,\Psi^{N}_{m}$ be the sequence
			of loss functions presented to Algorithm 1, satisfying the properties
			in Proposition \ref{Proposition 2.5}. Then,
			
			(i) The $w$-local regret incurred satisfies
			\begin{equation}
				\mathfrak{R}_{m,w}(N)\leq(\delta+\frac{2\kappa}{w})^{2}N.\label{Regret Bound}
			\end{equation}
			
			(ii) There exists $\eta^{\prime}\leq\eta^{0}$ such that the total
			number of gradient steps taken by Algorithm 1 is bounded by
			\begin{equation}
				\frac{2B}{\delta^{2}\left(\|\eta^{\prime}\|_{\min}-\frac{\beta\,\|\eta_{0}\|_{\infty}^{2}}{2}\right)w}N,\label{Gradient I}
			\end{equation}
			where $B$ is the bound of the loss function stated in Assumption \ref{Assumption_1}.
		\end{thm}
}}
Theorem \ref{Theorem 4.1} illustrates that Algorithm 1 attains the local
regret bound $O(N/w^{2}+N/w+N)$ with $O(N/w)$ gradient evaluation
steps. Theorem \ref{Theorem 4.1} (i) illustrates
that ATGD has the average local regret $\mathfrak{R}_{w}(N)/N = O(1)$. {\color{black}{Such performance is thus acceptable, as mentioned previously in Section 2.4.}} Besides,
the bound (\ref{Gradient I}) is larger than the bound in \cite[Algorithm 1]{hazan_efficient_2017}
owing to the diminishing step sizes, which will lead to more gradient estimations of Algorithm 1 than \cite[Algorithm 1]{hazan_efficient_2017} but will lead to a lower local regret bound than \cite[Algorithm 1]{hazan_efficient_2017}
given the fixed number of gradient estimations. The results reveal
a tradeoff between low computational complexity and high convergence
guarantee. Moreover, our results show that for a fixed streaming
data size $N$, increasing the window size $w$, that is, increasing the
smoothing effect, will reduce both the complexity of the local regret and the
total number of gradient steps. {\color{black}{In terms of the lower bound on local regret, it is shown from \cite[Theorem 2.7]{hazan_efficient_2017} that for any online algorithm, an adversarial sequence of loss functions can force the
		local regret incurred to scale with $N$ as $\Omega(N/w^{2})$.}}

{\color{black}{The optimal window size $w^{\ast}$ can be determined when fixing computational resources (i.e., fixing the number of gradient estimations) such that the local regret
		bound can be minimized.}} By fixing the bound on the total number
of gradient steps as $T$, the required size of input data will
be
\[
N=\frac{\delta^{2}\left(\|\eta^{\prime}\|_{\min}-\frac{\beta\,\|\eta^{0}\|_{\infty}^{2}}{2}\right)w}{2B}T.
\]
The bound on the $w$-local regret will become
\[
\left[\delta^{2}w+4\delta\kappa+\frac{4\kappa^{2}}{w}\right]\frac{\delta^{2}\left(\|\eta^{\prime}\|_{\min}-\frac{\beta\,\|\eta^{0}\|_{\infty}^{2}}{2}\right)}{2B}T.
\]
\textcolor{black}{By minimizing the $w$-local regret in $w$, the optimal window size $w^\ast$ will equal $1$ if $2\kappa/\delta\leq1$ and $N$ if $2\kappa/\delta\geq N$. When $1<2\kappa/\delta<N$, the optimal window size $w^\ast$ will equal $2\kappa/\delta$.} Thus, the optimal window size
$w^{\ast}$ is dependent on the Lipschitz-continuous modulus of function
$\Psi_{n}$ as well as the tolerance level $\delta$, but is independent
on the gradient step sizes as well as the computational budget. To
interpret this, the more ``fluctuation'' the shape of loss functions has and
the more accurate the estimation requirement on the algorithm is, the
larger the required window sizes that should be chosen. In another words,
the optimal window size
should be determined based on the specific properties of the pollution source
(choice of ADE model) and the parameters of the river (choice of ADE model parameters). 
\begin{rem}
	(i) Suppose the step sizes are {\color{black}{constant}}
	in Algorithm 1, where step sizes $\{\eta^{n}\}$ are replaced with
	$\eta^{0}$ and the Backtracking-Armijo line-search step will not be implemented.
	Then, Theorem \ref{Theorem 4.1} will still hold for the new algorithm
	by replacing $\eta^{\prime}$ with $\eta^{0}$ in the bound (\ref{Gradient I}).
	When $\{\eta^{n}\}$ becomes a {\color{black}{constant}} and uni-variant value, Algorithm 1 will become \cite[Algorithm 1]{hazan_efficient_2017}. Thus, our Algorithm
	1 generalizes \cite[Algorithm 1]{hazan_efficient_2017}. 
	
	(ii) From the proof of Theorem \ref{Theorem 4.1}
	in Appendix C, the theoretical guarantee of Algorithm
	1 summarized in the theorem is independent of the structure of the inner
	loop updating the step sizes. In other words, Algorithm 2 can be replaced with other adaptive step-size approaches (e.g., the Backtracking-Armijo line search, Generic line search, and Exact line
	search) only if those methods continue to reduce the step sizes when the stopping criterion is not satisfied. Thus, an analytic framework for gradient-based online non-convex optimization algorithms is developed with adjusted step sizes by the line search.
\end{rem}
{\color{black}{\subsection{Sequential Sensor Installation Mechanism}}}
{\color{black}{		
In this section, a case where a measurement error exists in the computed identification results in each iteration $n\in[N]$ is considered. In real life, the error can consist of \emph{systematic error} from the measurements of sensors and \emph{random error} from the experimenter's interpretation of the instrumental reading and unpredictable environmental condition fluctuations. Normal (Gaussian) distribution is the most commonly used distribution for the measurement error, both in theory and practice (see \cite{chesher1991effect, fuller1995estimation, kelly2007some, ruotsalainen2018error, coskun2020statistical}). Assume the measurement error follows multivariate normal distribution among sensors in each iteration $n\in[N]$, and the distribution is stationary for all $n\in[N]$. However, the explicit distribution and its mean and variance are not known. Installing multiple sensors can help reduce the measurement error, but implementing such sensors is an extremely costly activity, including costs for manpower, equipment, and chemical reagents (see \cite{chinrungrueng2006vehicular,carlsen2008using,park2012three,abubakar2015overview,dong2017deployment}). Thus, there is a trade-off between the costs and measurement accuracy requirements. The goal of the section is to find the minimal number of sensors possible such that the measurement error can be controlled. 

The confidence interval for the source identification result is first derived when $|M|$ is given. As the normal distribution is assumed, and standard deviation is unknown, the confidence interval is constructed based on the Student's $t$-distribution theory. Denote $x^{n}_{m} = (s^{n}_{m},\,l^{n}_{m},\,t^{n}_{m})$ as the identification results output from Algorithm 1 at sensor $m\in M$ using the first $n$th iterations. The mean across all sensors can be computed as \[\bar{x}^{n} = (\bar{s}^{n},\,\bar{l}^{n},\,\bar{t}^{n}) = \frac{1}{|M|} \sum_{m\in M} x^{n}_{m} = (\frac{1}{|M|}\sum_{m\in M} s^{n}_{m},\,\frac{1}{|M|}\sum_{m\in M} l^{n}_{m},\,\frac{1}{|M|}\sum_{m\in M} t^{n}_{m}).\] 
{\color{black}{The standard deviation on each dimension is $S_{s^{n}} = \sqrt{\frac{1}{|M|-1} \sum_{m\in M} (s^{n}_{m}-\bar{s}^{n})^{2}}$, $S_{l^{n}} = \sqrt{\frac{1}{|M|-1} \sum_{m\in M} (l^{n}_{m}-\bar{l}^{n})^{2}}$ and $S_{t^{n}} = \sqrt{\frac{1}{|M|-1} \sum_{m\in M} (t^{n}_{m}-\bar{t}^{n})^{2}}$, respectively.}} The next proposition derives the confidence interval for the released mass identification result as an example. The confidence intervals for the location and time identification results can be derived using the same idea.
\begin{prop}
	(\cite{sutradhar_characteristic_1986,taboga_lectures_2017,prins_nistsematech_2013}) Given the set of sensors $M$, the $(1-\alpha)$ confidence interval for the released mass identification result after $n$ iterations is 
	\begin{equation}\label{interval}
		\left[\bar{s}^{n} - t_{\alpha/2,\,|M|-1} \frac{S_{s^{n}}}{\sqrt{|M|}},\, \bar{s}^{n} + t_{\alpha/2,\,|M|-1} \frac{S_{s^{n}}}{\sqrt{|M|}}\right],
	\end{equation}
	where $t_{\alpha/2,\,|M|-1}$ is the two-sided $t$-value for $1-\alpha$ confidence under $|M|-1$ degree of freedom. This value can be found from the standard table in \cite[Chapter 1.3.6.7.2.]{prins_nistsematech_2013}. 
\end{prop}
Suppose $d^{s},\,d^{l}$ and $\,d^{t}$ are the acceptable length of confidence intervals, under $1-\alpha$ confidence on each dimension, respectively, set by the decision maker. These also quantify the estimation accuracy requirement. Given formulation Equation (\ref{interval}), a natural idea to compute the minimal sensors required on each dimension is to set $d^{s},\,d^{l},\,d^{t}$ equal to the lengths of the corresponding confidence intervals. For instance, $d^{s} = 2t_{\alpha/2,\,|M|-1} \frac{S_{s^{n}}}{\sqrt{|M|}}$. The $|M|$ satisfying the equation is the minimal number of sensors required to ensure estimation accuracy at the concentration level. The drawback is that the two-sided $t$-value depends on known degrees of freedom, which in turn depends upon the sample size that we are trying to estimate. To solve this problem, a power analysis procedure is proposed, starting with an initial estimate based on a sample standard deviation and iteration (see \cite[Section 7.2.2.2.]{prins_nistsematech_2013}). However, this power analysis still has a drawback in real life applications because it may output a required number of sensors that is more than the number of sensors that exist in reality, and these extra samples cannot be collected. 

A sequential method is proposed, as summarized in Algorithm 4. In the algorithm, first, Partition $N$ iterates with several rounds, and each contains the $T$ iteration. The total rounds are $\lfloor N/T \rfloor +  1$. The number of iterations in the final rounds is the remainder. Suppose $A$ sensors are installed initially; that is, $|M_{0}| = A>0$ (e.g., $A=2$). Then, at each round $i=1,...,\lfloor N/T \rfloor +  1$, only the data within that round are used for identification (the identification results in previous rounds will be deleted, and the online learning algorithms will again have a cold start). Choose the data from $n \in \{2,...,|M_{i-1}|\}$ sensors. If the updated required sensor number $|\tilde{M_{i}^{s}}|$ is larger than the current one $|M_{i-1}|$, then the algorithm will jump out of the current round, and the extra sensors will be installed up to $|\tilde{M_{i}^{s}}|$. Otherwise, a new $|\tilde{M_{i}^{s}}|$ will be updated by the bisection method. The required number of sensors for the released mass $s$ is the minimal value $|\tilde{M_{i}^{s}}|$ across all $n \in \{2,...,|M_{i-1}|\}$. The same procedures will be applied to estimate the location ($l$) and time ($t$). Then, the required number of sensors used for the next round $i+1$ will be the maximum across all dimensions.

\begin{algorithm}
	\caption{Sequential Sensor Installation Design}
	
	\textbf{Input}: Partition $N$ iterations with several rounds, and each contains $T$ iteration. The total number of rounds is $\lfloor N/T \rfloor +  1$. The number of iterations in the final rounds is the remainder. Initialize $|M_{0}| = A>0$;
	
	\textbf{for} $i=1,...,\lfloor N/T \rfloor +  1$ \textbf{do}
	
	$\quad$Denote the initial required sensor size for the concentration level at round $i$ to be $|M_{i-1}|$. 
	
	$\quad$\textbf{for} $n=2,...,|M_{i-1}|$ \textbf{do}
	
	$\quad\quad$Set $|M_{i}^{s}| = n$. Arbitrarily set $\hat{|M_{i}^{s}|} \neq |M_{i}^{s}|$.
	
	$\quad\quad$\textbf{while} $\hat{|M_{i}^{s}|} \neq |M_{i}^{s}|$, \textbf{do}
	
	$\quad\quad\quad$Compute $\bar{s}^{T} = \frac{1}{|M_{i}^{s}|}\sum_{m\in M_{i}^{s}} s^{T}_{m}$ where $s^{T}_{m}$ is the identification results at the end of iteration $T$, only using the data in round $i$. 
	
	$\quad\quad\quad$Compute $S_{s^{T}} = \sqrt{\frac{1}{|M_{i}^{s}|-1} \sum_{m\in M_{i}^{s}} (s^{T}_{m}-\bar{s}^{T})^{2}}$.
	
	$\quad\quad\quad$Update $\hat{|M_{i}^{s}|} = \lceil\left(t_{\alpha/2,\,|M_{i}^{s}|-1}~S_{s^{T}}/d^{s}\right)^{2}\rceil$.	
	
	$\quad\quad\quad$Update $|\tilde{M_{i}^{s}}| =\lceil\left(|M_{i}^{s}|+\hat{|M_{i}^{s}|}\right)/2\rceil$. 
	
	$\quad\quad\quad$\textbf{if} $|M_{i-1}|<|\tilde{M_{i}^{s}}|$, \textbf{then}
	
	$\quad\quad\quad\quad$\textbf{break}
	
	$\quad\quad\quad$Set $|M_{i}^{s}| = \tilde{|M_{i}^{s}|}$.
	
	$\quad\quad$\textbf{end while}
	
	$\quad\quad$Set $R^{s}_{i,n} = \tilde{|M_{i}^{s}|}$, where $R^{s}_{i,n}$ records the number of sensors w.r.t to concentration $s$. 
	
	$\quad$\textbf{end for}
	
	$\quad$Using the similar procedures (in the above \textbf{for} loop) w.r.t location $l$ and time $t$. Set \[|M_{i}| = \max\{\min_{n}\{R^{s}_{i,n}\},\,\min_{n}\{R^{l}_{i,n}\},\,\min_{n}\{R^{t}_{i,n}\}\}.\]

	\textbf{end for}
\end{algorithm}
}}

\section{Escaping from Saddle Points}
In this section, the basic algorithm
(Algorithm 1) is further extended in terms of its practical implementation by deriving
the module \textquotedblleft escaping from saddle points\textquotedblright{},
which helps the algorithm arrive at a local minima with a high probability
instead of being stuck in any first-order stationary point. We also show the cumulative regret of the algorithm under an error bound
condition on the loss functions.

In non-convex optimization, the convergence to first-order stationary
points (points where first-order derivatives equal zero) is not satisfactory.
For non-convex functions, first-order stationary points can be global
minima, local minima, saddle points, or even local maxima. For many
non-convex problems, it is sufficient to find a local minimum. In
each iteration ($n$) in our basic algorithm ATGD, the algorithm will
terminate when $x^{n}_{m}$ is the $\delta$-first-order stationary point
of the follow-the-leader iteration $F^{n}_{m,w}$ on $\mathcal{F}$. However,
such a point may not necessarily be the local minimum of $\Psi^{n}_{m}$. Additional techniques are required to escape all saddle points
and arrive at the local minima. 

{\color{black}{
Based on Proposition \ref{Proposition 2.5}(iii) and triangle inequality, it can be shown that $\nabla_{x}^{2} F^{n}_{m,w}(x)$ is also $\iota$-Hessian Lipschitz. The definitions of the second-order stationary point and $\epsilon$-second-order stationary point are given.

\begin{defn}
	\cite{nesterov_cubic_2006,jin_how_2017} For the $\iota$-Hessian Lipschitz function $F^{n}_{m,w}$, the point $x$ is a second-order stationary point if $\|\nabla_{x} F^{n}_{m,w}(x)\|= 0$ and $\lambda_{\min}\left(\nabla_{x}^{2}F^{n}_{m,w}(x)\right)\geq 0$. The point $x$ is a $\epsilon$-second-order stationary point if
	\[
	\|\nabla_{x} F^{n}_{m,w}(x)\|\leq\epsilon,\quad\textrm{and}\quad\lambda_{\min}\left(\nabla_{x}^{2}F^{n}_{m,w}(x)\right)\geq-\sqrt{\iota\,\epsilon}.
	\]
\end{defn}

Second-order stationary points are important in non-convex
optimization because when all saddle points are strict ($\lambda_{\min}\left(\nabla_{x}^{2} F^{n}_{m,w}(x)\right)<0$),
all second-order stationary points are exactly the local minima. The goal here is to design algorithms that converge to the $\epsilon$-second-order stationary point, provided the threshold $\epsilon>0$.

A robust version of the ``strict saddle'' condition is next provided {\color{black}{and holds}} for function $F^{n}_{m,w}$ based on \cite{ge_escaping_2015,jin_how_2017}.
}}
\begin{prop}
\label{Proposition 4.7} Function $F^{n}_{m,w}(\cdot)$
is $(\theta,\,\tau,\,\zeta)$-strict saddle with $\theta,\,\tau>0$
and $\zeta\geq0$. That is, for any $x$, at least one of the following holds: (i) $\|\nabla_{x}F^{n}_{m,w}(x)\|\geq\theta$; (ii)
$\lambda_{\min}\left(\nabla_{x}^{2}F^{n}_{m,w}(x)\right)\leq-\tau$;
(iii) $x$ is $\zeta$-close to the set of the local minima of $F^{n}_{m,w}$.
\end{prop}

Algorithm 5 is a perturbed form of the gradient descent algorithm. For each iteration
$n$ in Algorithm 5, the algorithm will search for the solution where
the current gradient is small ($\leq\delta$) (which indicates that
the current iteration $x_{t}^{n}$ is potentially near a saddle point),
the algorithm adds a small random perturbation to the gradient. The
perturbation is added at most only once every $t_{\textrm{thres}}$
iterations. If the function value does not decrease enough (by $f_{\textrm{thres}}$)
after $t_{\textrm{thres}}$ iterations, the algorithm will output
$\tilde{x}^{n+1,t_{\textrm{noise}}}$. This can be proven to be necessarily
close to a local minima. The step-size choice operator $\mathfrak{S}(F^{n}_{m,w},\,K)$, instead of requiring the condition $\|\eta^{0}\|_{\min}/\|\eta^{0}\|_{\infty}^{2}\gg \beta/2$ in Algorithm 1, requires $\|\eta^{0}\|_{\infty} = c/\kappa$ in Algorithm 5.

\begin{algorithm}
\caption{Adaptive perturbed time-smoothed online gradient descent (APTGD)}

\textbf{Input}: sensor $m\in M$, window size $w\geq1$, Constant $c\leq1$,
tolerance $\delta>0$, contant $K>0$, and a convex set $\mathcal{F}$;

\textbf{Input:} $\mathcal{\chi}=3\max\left\{ \log\left(\frac{d\kappa\triangle f}{c\delta^{2}\epsilon}\right),\,4\right\} $,
$r=\frac{\sqrt{c}}{\chi^{2}}\cdot\frac{\delta}{\kappa}$, $g_{\textrm{thres}}=\frac{\sqrt{c}}{\mathcal{\chi}^{2}}\cdot\delta$,
$f_{\textrm{thres}}=\frac{c}{\mathcal{\chi}^{3}}\sqrt{\frac{\delta^{3}}{\iota}}$,
$t_{\textrm{thres}}=\lceil\frac{\chi}{c^{2}}\cdot\frac{\kappa}{\sqrt{\iota\,\delta}}\rceil$,
$t_{\textrm{noise}}=-t_{\textrm{thres}}-1$

\textbf{Set} $x^{1}_{m}\in\mathcal{F}$ arbitrarily.

\textbf{for} $n=1,\,...,\,N$ \textbf{do}

$\quad$Observe the cost function $\Psi^{n}_{m}:\,\mathcal{F}\rightarrow\mathbb{R}$.

$\quad$Compute the initial step size: $\eta^{0} = \mathfrak{S}(F^{n}_{m,w},\,K)$. 

$\quad$Initialize $x^{n+1}_{m}:=x^{n}_{m}$ and $t_{\textrm{noise}}$;

$\quad$\textbf{for} $t=0,\,1,\,...$ \textbf{do}

$\quad\quad$Initialize $x^{n+1,t}:=x^{n+1}_{m}$;

$\quad\quad$Determine $\eta^{t}$ using Normalized Backtracking-Armijo line
search (Algorithm 2).

$\quad\quad$ \textbf{if} $\|\nabla_{\mathcal{F},\eta^{t}}F^{n}_{m,\,w}(x^{n+1,t})\|\leq g_{\textrm{thres}}$
and $t-t_{\textrm{noise}}>t_{\textrm{thres}}$ \textbf{then}

$\quad\quad\quad$$t_{\textrm{noise}}=t$, $\tilde{x}^{n+1,t}=x^{n+1,t}$,
$x^{n+1,t}=\tilde{x}^{n+1,t}+\omega$ where $\omega$ uniformly
sampled from $\mathbb{B}_{0}(r)$

$\quad\quad$ \textbf{if} $t-t_{\textrm{noise}}=t_{\textrm{thres}}$ and $F_{n,\,w}(x^{n+1,t})-F_{n,\,w}(\tilde{x}^{n+1, t_{\textrm{noise}}})>-f_{\textrm{thres}}$\textbf{
then}

$\quad\quad\quad$\textbf{Return} $x^{n+1}_{m}=\tilde{x}^{n+1, t_{\textrm{noise}}}$

$\quad\quad$Update $x^{n+1, t+1}=x^{n+1,t}-\eta^{t}\otimes\nabla_{\mathcal{F},\eta^{t}}F^{n}_{m,\,w}(x^{n+1,t})$.

$\quad$\textbf{end for}

\textbf{end for}
\end{algorithm}

\begin{thm}
\label{Theorem 4.8 } Let $\Psi^{n}_{m}$ satisfy properties in Proposition
\ref{Proposition 2.5} and \ref{Proposition 4.7}. There exists an
absolute constant $c_{max}\leq1$ such that for $c\leq c_{max}$,
and $\triangle f\geq\max_{n=1,\,...,\,N}\left\{ F^{n}_{m,w}(x^{n-1}_{m})\right\} $
by letting $\delta:=\min\left(\theta,\,\tau^{2}/\iota\right)$, Algorithm 5 will output a sequence of $\delta$-second-order stationary points
that are also the $\zeta$-close local minima of $\{F^{n}_{m,\,w}\}_{n=1,\,...,\,N}$, with the
probability $1-\epsilon$, and terminate in the following total number
of iterations (gradient estimations): 
\begin{equation}
O\left(\frac{\kappa\triangle f}{\delta^{2}}\log^{4}\left(\frac{d\kappa\triangle f}{\delta^{2}\,\epsilon}\right) N\right). \label{Regret bound 2}
\end{equation}
\end{thm}

The proof of Theorem \ref{Theorem 4.8 } is shown in Appendix D, which
follows the results in \cite{jin_how_2017}. Here, their results on the standard gradient descent can be naturally extended to the online gradient descent, as our basic algorithm contains an offline follow-the-leader oracle
at each period. Theorem \ref{Theorem 4.8 } shows that by careful
selection of $\triangle f$ and $\delta$, the algorithm will output the
local minima of each follow-the-leader oracle in each period with
a linear number of iterations in a high probability. As
mentioned in \cite{jin_how_2017}, Theorem \ref{Theorem 4.8 } only explicitly
asserts that the output will lay within some fixed radius $\zeta$
from a local minimum. In many real applications, $\zeta$ can be further written as a function $\zeta(\cdot)$ of the gradient threshold $\theta$ such that $\theta$ decreases, $\zeta(\theta)$ decreases linearly
or polynomially depending on $\theta$. The following Definition \ref{Definition 4.3} is given, based on which Example \ref{Example 4.4} is given. {\color{black}{Given the complexity of the objective $F^{n}_{m,w}$, the Example \ref{Example 4.4} ensures that  $\zeta(\theta)$ decreases linearly on $\theta$ with a high probability. }}
{\color{black}{
\begin{defn} \label{Definition 4.3} (Restricted strongly convex) \cite[Definition 1 and Definition 2]{zhang_restricted_2017}
	Let the nonempty set $\mathcal{X}\subset \mathcal{F}$ be the set of all local minimum of $F^{n}_{m,w}$. Then, function $F^{n}_{m,w}$ is restricted and strongly convex with the constant $\mu>0$ if it satisfies the restricted secant
	inequality
	\[\langle \nabla_{x}F^{n}_{m,w}(x),\, x - \Pi_{\mathcal{X}}[x] \rangle \geq \mu~d^{2}(x,\,\mathcal{X}),
	\]
	where $d(x,\,\mathcal{X})$ measures the distance from a point $x$ to set $\mathcal{X}$. 
\end{defn}

\begin{example}\label{Example 4.4} (Local error bound) Based on \cite{luo_error_1993,drusvyatskiy_error_2018}, 
	and assuming that restricted strongly convexity in Definition \ref{Definition 4.3} holds when $x$ lays within the fixed radius $\zeta$ from a local minimum, the function $F^{n}_{m,w}$ is strongly convex in the $\zeta$-ball of any local minimum, and then the \emph{local error bound} condition holds that 
	\[d(x,\,\mathcal{X}) \leq \mu \|\nabla_{x}F^{n}_{m,w}(x)\|,
	\]
	for all $n=1,...,N$. Thus, in this example, $\zeta$ decreases linearly on $\theta$ with the rate $\mu>0$; that is, $\zeta(\theta) = \mu \theta$. Suppose Algorithm 5 outputs the solution $x_{m}^{n}$ at each iteration $n$. \cite{luo_error_1993} finds that if $F^{n}_{m,w}$ is strongly convex on $\mathcal{F}$, then the local error bound condition holds for all $x\in\mathcal{F}$. The following steps provide a valid way of checking if the result holds with a high probability, and the main idea is to check the strong convexity of $F^{n}_{m,w}$ on the $\zeta$-domain of $x_{m}^{n}$: 
	
	\textbf{Step 1.} Compute the Hessian $\nabla_{x}^{2} F^{n}_{m,w}(x_{m}^{n})$ and estimate $\mu$ such that $\mu$ is a lower bound for the smallest eigenvalue of the Hessian. 
	
	\textbf{Step 2.} Simulate a set of points in the $\mu\theta$-domain of $x_{m}^{n}$ and check if the Hessian of $F^{n}_{m,w}$ at those points are positive definite. If so, then the local error bound holds for function $F^{n}_{m,w}$ with a high probability. 
\end{example}

For the above Step 2, instead of checking the positive definiteness of all the points in the $\mu\theta$-domain of $x_{m}^{n}$, finite points are generated and checked (see \cite[Theorem 3]{milne_piecewise_2018}) such that the local error bound condition holds with a high probability. 

Next, it can be shown that the local regret can be linked with the cumulative regret of Algorithm
5 under specific error bound conditions. The local error bound condition on function $\{\Psi^{n}_{m}\}_{n=1,...,N}$ is required, as summarized by the following assumption, where $w = 1$.
\begin{assumption} \label{Assumption 4.5}
Let $x^{n,\ast}_{m}$ be any local minimum of the function $\Psi^{n}_{m}$; then, 
\[\|\nabla_{x}F^{n}_{m,\,1}(x)\| = \|\nabla_{x}\Psi^{n}_{m}(x)\| \geq\frac{\mu}{\kappa}\left[\Psi^{n}_{m}(x)-\Psi^{n}_{m}(x_{m}^{n,\ast})\right],\]
when $x$ lays within the fixed radius $\zeta$ from $x_{m}^{n,\ast}$. There also exists $x^{n,\ast}_{m}$ such that $x^{n}_{m}$ satisfies the condition. 
\end{assumption}
To check if $\Psi^{n}_{m}$ satisfies Assumption \ref{Assumption 4.5}, the following steps are required, which follow the similar idea in Example \ref{Example 4.4}. Owing to the complexity of the objective $\Psi^{n}_{m}$, this method only ensures that Assumption \ref{Assumption 4.5} holds with a high probability: 

\textbf{Step 1.} Solve the local minimum of $\Psi^{n}_{m}(x)$ denoted by $x^{n,\prime}_{m}$ by Algorithm 5. 

\textbf{Step 2.} Compute the Hessian $\nabla_{x}^{2} \Psi^{n}_{m}(x^{n,\prime}_{m})$ and estimate $\mu$ such that $\mu/\kappa$ is a lower bound for the smallest eigenvalue of the Hessian. 

\textbf{Step 3.} Simulate a set of points in the $\mu\theta$-domain of $x_{m}^{n,\prime}$ and check if the Hessian of $\Psi^{n}_{m}$ at those points are positive definite. Then, check if $x^{n}_{m}$ lays within the $(\mu\theta)/\kappa$-domain of $x^{n,\prime}_{m}$.

The following Theorem \ref{Theorem 4.9 } provides a high
probability bound for the cumulative regret of Algorithm 5 with $O(N)$; the proof is given in Appendix D.
\begin{thm}
\label{Theorem 4.9 } Let $\Psi_{m}^{n}$ satisfy Assumption \ref{Assumption 4.5}. Then, with probability is at most $1-\epsilon$, and the cumulative regret of Algorithm
5 after
\[
O\left(\frac{\kappa\triangle f}{\delta^{2}}\log^{4}\left(\frac{d\kappa\triangle f}{\delta^{2}\,\epsilon}\right)\cdot N\right)
\]
gradient estimations, with $\triangle f\geq\max_{n=1,\,...,\,N}\left\{ F^{n}_{m,w}(x^{n-1}_{m})\right\} $
and $\delta:=\min\left(\theta,\,\tau^{2}/\iota\right)$, is bounded
by
\begin{equation}
O\left(\frac{\mu \delta}{\kappa}\cdot N\right). \label{Cumulative regret}    
\end{equation}
\end{thm}
}}
\begin{rem}
For practical implementation, the multi-start version algorithms can be developed for both ATGD and APTGD. The main idea is to create several solution paths when
implementing the algorithms given different starting points. In each iteration, the square of the difference between the ADE value and real concentration data for each sensor is recorded. Following the idea from \cite{jiang_integration_2015}, the minimizer of the average of the square of the difference across all sensors, among all
solution paths, is chosen, and we treat it as the estimated optimal solution at that iteration. Practically, the multi-start implementation can increase the chance of finding more accurate source information estimation. The explicit structures of multi-start algorithms are shown as Algorithm 7 (Multi-start time-smoothed online gradient descent, or MTGD) and Algorithm 8 (Multi-start perturbed time-smoothed online gradient descent, or MPTGD) in Appendix A. 
\end{rem}

\section{Experiments and Application}
In this section, the experiments validating the theoretical result of our online learning algorithms are conducted. Our algorithms are also applied to real-life river pollution source identification problems. 
For the experiments, the Rhodamine WT dye concentration data from a travel time study on
the Truckee River between Glenshire Drive near Truckee, California and
Mogul, Nevada \citep{crompton_traveltime_2008} are used. Figure 2 presents locations of a pollutant site,
marked by a black star, and four sampling sites (sensors), marked by red stars.
{\color{black}{The data include the real data: the true pollutant source information
$(s,\,l,\,t) = [1300,-22106,-215]$ and 82 pairs of concentration data detected in the four downstream sampling sites, and 1000 pairs of artificial data generated by the ADE model under the the true pollutant source information. }} The parameters in the ADE model are obtained from \cite{crompton_traveltime_2008} and \cite{jiang_applicability_2017}; namely, $v=80m/min, D = 2430m^2/min, A = 60m^2, $ and $ k = 10^{-8}min^{-1}$. 

\begin{figure}[H]
\protect\begin{centering}
\protect\includegraphics[scale=0.5]{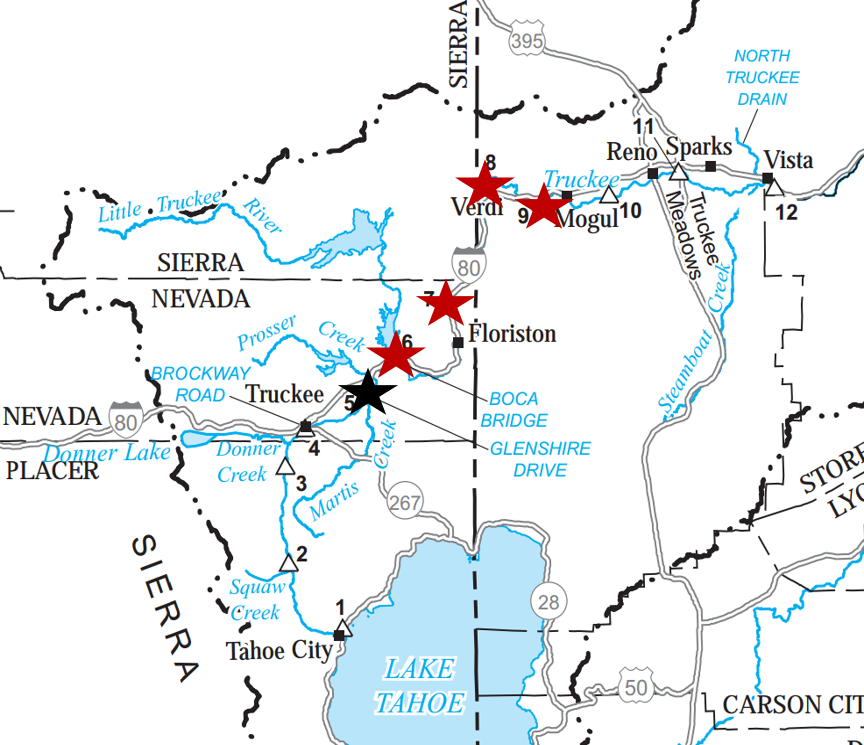}\protect
\par\end{centering}
\protect\caption{Locations of the pollutant site and sampling sites on the Truckee River}
\end{figure}

\subsection{Regret and Optimality Analysis}
In this experiment, {\color{black}{1000 pairs of artificially generated concentration data}} are used to compute the local regret bound (\ref{Local regret}) and the cumulative regret bound (\ref{Cumulative}) for various online algorithms, including: time-smoothed online gradient descent \cite[Algorithm 1]{hazan_efficient_2017} (called ``TGD'' in short), ATGD, APTGD, MTGD, and MPTGD. TGD is explicitly presented as Algorithm 6. {\color{black}{The step size is chosen to be $\eta = 10$ throughout all the experiments in the following sections.}}

\begin{algorithm}
\caption{TGD}

\textbf{Input}: Sensor $m\in M$; window size $w\geq1$, learning rate $0<\eta\leq \beta/2$; tolerance $\delta>0$, constant $K>0$, and a convex set $\mathcal{F}$; 

\textbf{Set} $x^{1}_{m}\in\mathcal{F}$ arbitrarily

\textbf{for} $n=1,\,...,\,N$ \textbf{do}

$\quad$ Predict $x^{n}_{m}$. Observe the cost function $\Psi^{n}_{m}:\,\mathcal{F}\rightarrow\mathbb{R}$.

$\quad$ Initialize $x^{n+1}_{m}:=x^{n}_{m}$. 

$\quad\quad$ \textbf{while} $\|\nabla_{\mathcal{F},\eta} F^{n}_{m,\,w}(x^{n+1}_{m})\|_{2}>\delta/w$
\textbf{do}

$\quad\quad\quad$ Update $x^{n+1}_{m}:=x^{n+1}_{m}-\eta\nabla_{\mathcal{F},\eta} F^{n}_{m,\,w}(x^{n+1}_{m})$.

$\quad\quad$\textbf{end while}

\textbf{end for}
\end{algorithm}}}

Choose the {\color{black}{window size $w=1$}}. The group size (number of multi-starts) is chosen as $I = 30$ in the multi-start algorithms. The initial step sizes are chosen by the method summarized in Section 3.1. The experimental results are recorded from Figure 3. Figure 3(a) validates that all of those algorithms have the local regret bound $O(N)$ as the inequality (\ref{Regret Bound}), and they also show experimentally that the cumulative regret of all those algorithm has the complexity $O(N)$, which potentially conforms to the bound (\ref{Cumulative regret}). {\color{black}{In addition, when all the algorithms are implemented, an extreme small tolerance $\delta=5\times 10^{-6}$ is set. Thus, in each iteration, the algorithms always stop close to a stationary point (can be either the saddle point or local minimum) for the current follow-the-leader iteration. So the local regrets of all the algorithms will be extremely low (within $2\times 10^{-5}$). {\color{black}{Figure 3(a) shows that the vectorized and adaptive step-size algorithms slightly improve the accuracy of the stationary point search in each iteration compared to the fixed and univariate step-size algorithms. Particularly for MTGD and MPTGD, the local regret increases slightly as the number of iterations increases.}}

From the plots of cumulative regret, it can be observed that adding the ``escaping from saddle points'' module can reduce the growth rate of the cumulative regrets, which means that as the size of data grows, the algorithms output a solution closer to the global minimizer, that is, the true source information. Figure 3(b) shows that multi-start algorithms (MTGD and MPTGD) can further reduce the cumulative regrets. {\color{black}{MPTGD has a significantly better cumulative regret than all other algorithms. Intuitively, the potential reason is the aggregated effect of the ``perturbed'' module and the multi-start module. The ``perturbed'' module increases the probability of convergence to a local minimum. The multi-start module ensures that the probability that the global minimum is among the multiple local minimums will be high. }}

\begin{figure}[!h]
\subfloat[]{
\begin{minipage}[c][1\width]{
0.5\textwidth}
\centering
\includegraphics[scale=.45]{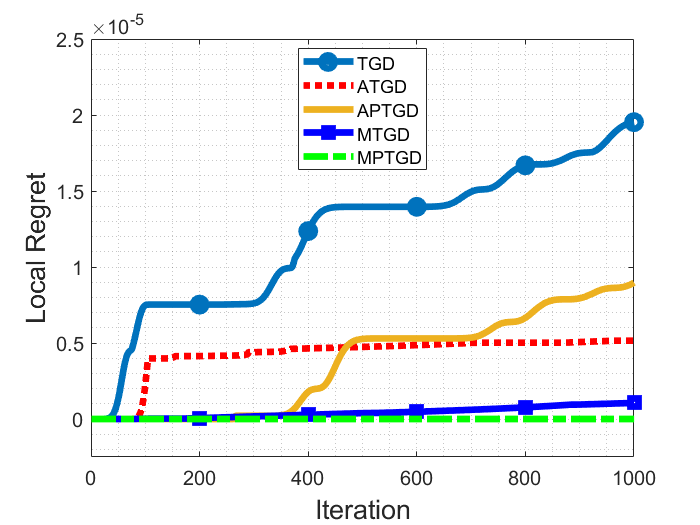}
\end{minipage}}
\subfloat[]{
\begin{minipage}[c][1\width]{
0.5\textwidth}
\centering
\includegraphics[scale=.45]{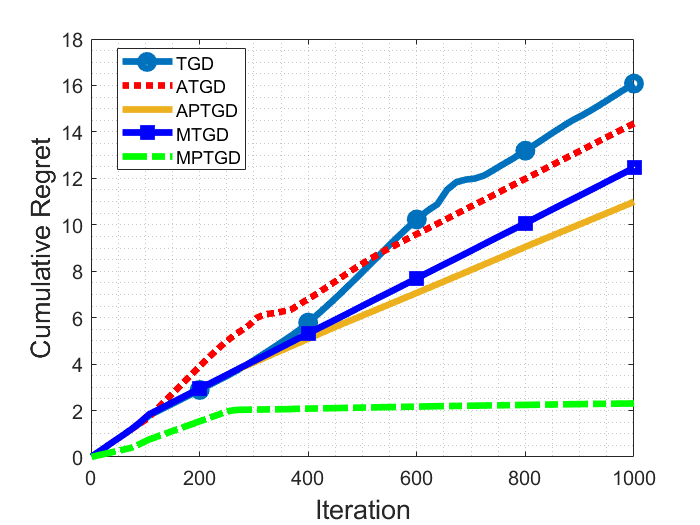}
\end{minipage}}
\caption{(a) Local regrets and (b) Cumulative regrets}
\end{figure}

{\color{black}{The next experiment is conducted to show the faster convergence of the proposed algorithms compared to TGD. In each iteration $n$, all of these algorithms require iteratively updating the $x_{m}^{n}$ until the norm of gradient $\|\nabla_{\mathcal{F},\eta}F^{n}_{m,\,w}(x^{n+1}_{m})\|_{2}$ (where $\eta$ is the updated step size at iteration $n$ and varies from different algorithms) is bounded by a certain threshold (for instance, the ``while'' loop in ATGD describes this process, and the threshold is $\delta$). Alternatively, this experiment implements the algorithms in a reverse setting. It compares the local regrets of all algorithms at the final iteration $N=1000$, given that in each iteration, the number of gradient evaluations is fixed. For MTGD and MPTGD, the number of gradient evaluations for each start $i\in[I]$ equals those of TGD, ATGD, and APTGD. Figure 4 validates that our proposed algorithms have a better local regret than TGD, given the same fixed number of gradient evaluations. For those algorithms to reach a similar local regret, TGD apparently needs more gradient evaluations. As gradient evaluation is normally the most time-consuming step in gradient-based algorithms, our proposed algorithms thus have better computational efficiency. It can also be observed that the perturbed algoritm (MPTGD and APTGD) have a higher local regret than their unperturbed versions (MTGD and ATGD). The possible reason is that the perturbed algorithms are escaping from saddle points when the gradient evaluations stops, so the norm of gradient is relatively high at that moment. 

\begin{figure}
\begin{center}
\includegraphics[scale=0.5]{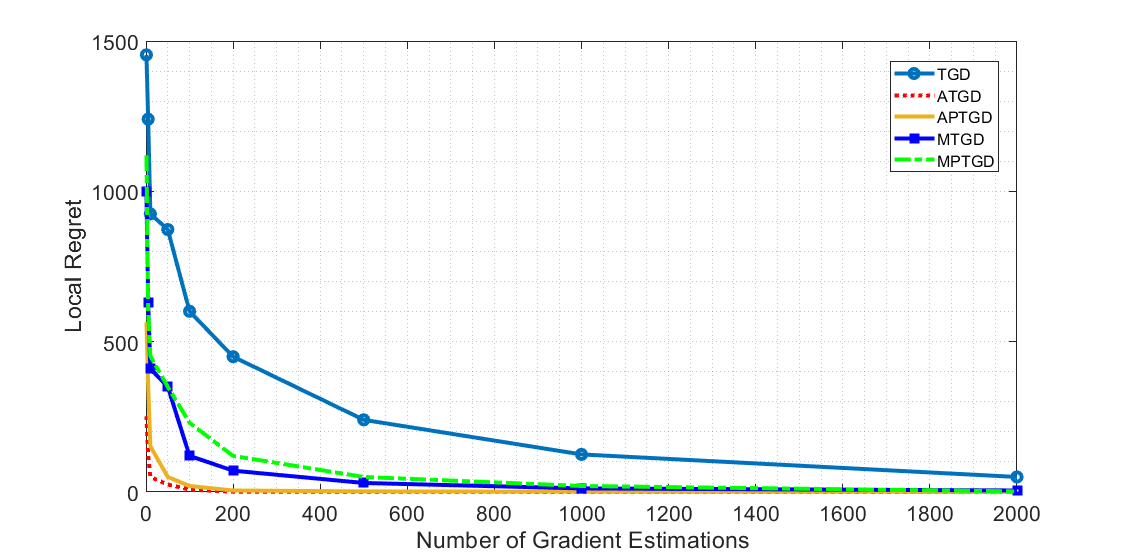}
\end{center}
\caption{Local regret comparison with fixed gradient estimation}
\end{figure} 

}}

{\color{black}{\subsection{Sequential Sensor Installation Analysis}}}

{\color{black}{In this section, the statistical analysis from Section 3.2 is implemented for the ATGD algorithm. Here, the true pollution source information is $(s,\,l,\,t) = [1300,-22106,-215]$, and the 1000 pairs of artificially generated data are used from sensors in 50 locations. Each round of Algorithm 4 contains 10 iterations, so there will be a total of 100 rounds. Assume that the measurement errors for concentration ($s$), location ($l$), and time ($t$) follow Normal distributions $\mathcal{N}(0,\,50)$, $\mathcal{N}(0,\,500)$ and $\mathcal{N}(0,\,5)$, respectively. The distribution is unknown to the decision maker when implementing Algorithm 4 and ATGD. To compare the performance of Algorithm 4, the ``true'' required number of sensors is computed by the identification results at the final iteration from $50$ sensors. Here the random measurement errors can be simulated from their given distributions. In the experiments, set $d^{s} = d^{t} = 200$ and $d^{l} = 500$. Figure 5 shows the minimal number of sensors required in each round and how the number is affected by the initial installation $A$. When $A=3$, the required number already becomes 20 in all later rounds. When $A>=19$, the required number will just be $A$ in later rounds. The optimal policy is to chose a moderate start number (e.g., $A = 7, 11, 13$), and the required number of sensors will finally become $19$. In this experiment, the ``true'' required number of sensors can be computed and equals 13. So the number of sensors determined by Algorithm 4 is less than two times the ``true'' number. In addition, the result is insensitive to $A$. By choosing a small number of sensors initially (e.g., $A=3,7$), the required number of sensors determined by Algorithm 4 will already be close to the optimum.
\begin{figure}
\begin{center}
\includegraphics[scale = .6]{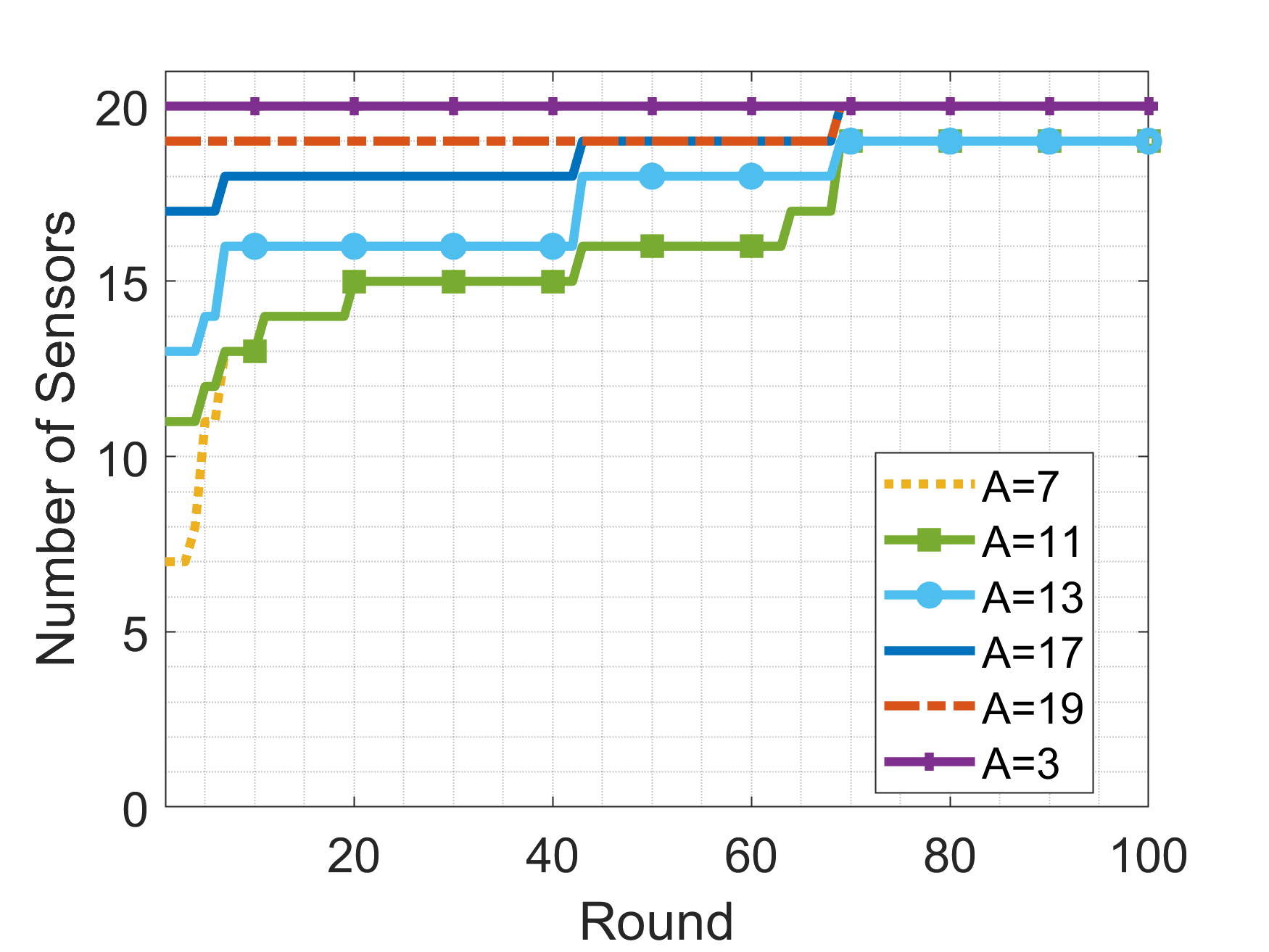}
\caption{Sequential Sensor installation}
\end{center}
\end{figure}
}}

\subsection{Identification Accuracy and Efficiency}
In this section, the computational results and the settings for various online algorithms on the real data are recorded (Table 1). The contents recorded include the estimation result for the pollution source, the estimation error to the true source information, and the average computation time. \textcolor{black}{The initial step size and the parameters in the Backtracking line search are recorded in Table 1. The number of multi-start paths is set to 30.} The result shows that \textcolor{black}{the relative error of estimation on each dimension of ATGD and MTGD is (3.46\%, 2.79\%, and 11.63\%). Compared with the estimation result of TGD (3.69\%, 4.63\%, and 14.42\%), ATGD and MTGD improve the estimation results on every dimension.} By incorporating the ``escaping from saddle points'' module, the error is reduced significantly on the released mass estimation, from 3.46\% to 1.31\%, and the error on the released time estimation is slightly reduced. However, the compensation is that the estimation error on the location estimation is increased to 11.35\%. The common shortcoming of TGD, ATGD, APTGD, and MTGD is the high estimation error on the released time estimation. In contrast, MPTGD overcomes this shortcoming by reducing the error significantly, to 1.40\%. MPTGD also has the longest computation time, as both the ``escaping from saddle points'' and multi-start modules are incorporated. 
Here ATGD, APTGD, MTGD, and MPTGD are all conducted under the initial step sizes $\eta=[110000,11000000,75000]$, which are computed based on the method introduced in Section 3. {\color{black}{In terms of computational time, ATGD is more than 10 times faster than TGD, and APTGD is $37\%$ faster than TGD.}}

Although from Table 1, it seems that there does not exist an algorithm in ATGD, APTGD, MTGD, and MPTGD that ``dominates'' all others in terms of estimation accuracy in each dimension, the decision-maker can choose to use each algorithm based on their own targets. However, for a real-life problem, the decision maker is recommended to use ATGD and MTGD if the decision maker has a high estimation accuracy requirement on the location estimation, APTGD if the decision maker has a high estimation accuracy requirement on the released mass estimation, and MPTGD if the decision maker has a high estimation accuracy requirement on the released time estimation. Thus, options and guidelines for the decision maker to choose the proper variant of algorithms are provided when they have emphasis on identification accuracy in different dimensions. 

\begin{table}[H]
\footnotesize
\caption{Computational results and settings}
\label{Computational Results}
\begin{center}
\renewcommand\arraystretch{0.8}
\begin{tabular}{p{2cm}p{3cm}p{3.5cm}p{3cm}p{2cm}}
\toprule
& $(s,l,t)$ & Relative Error &  Parameter Setting & Time (Second)\\
\midrule
TGD	& (1348, -23130, -184)	& (3.69\%, 4.63\%, 14.42\%)	&  $\beta=8\times 10^{-6}$ &1.1872\\
ATGD	& (1345, -22722, -190)	& (3.46\%, 2.79\%, 11.63\%)	& $\beta=8\times 10^{-6}$ & 0.1023
\\
APTGD	& (1317, -19597, -191) & (1.31\%, 11.35\%, 11.16\%) & $\beta=2\times 10^{-6}$ &
0.7471
\\
MTGD & (1345, -22722, -190) & (3.46\%, 2.79\%, 11.63\%) & $\beta=8\times 10^{-6}$, $I = 30$ & 6.0764\\
MPTGD & (1392, -20994, -212) & (7.00\%, 4.78\%, 1.40\%) & $\beta=2\times 10^{-6}$, $I = 30$ & 21.1709 \\
\bottomrule
\end{tabular}
\end{center}
\end{table}

In addition, the empirical analysis for the number of multi-starts is conducted. Table 2 shows that the relative error on the estimation of the released location ($l$) will decrease significantly as $I$ increases, and the computation time increases linearly as $I$ increases. Table 3 shows that the relative error on the estimation of the released mass ($s$) and location ($l$) decreases as $I$ increases, and the computation time also increases linearly as $I$ increases.

\begin{table}[H]
\footnotesize
\caption{Multi-start for MTGD}
\label{Computational Results}
\begin{center}
\renewcommand\arraystretch{0.8}
\begin{tabular}{p{2cm}p{3.5cm}p{3.5cm}p{2cm}}
\toprule
& $(s,l,t)$ & Relative Error & Time (Second)\\
\midrule
$I = 5$ & (2105, -29761, -166)  & (61.92\%, 34.63\%, 22.79\%) & 2.46 \\
$I = 10$  & (1315, -29493, -364) & (1.15\%, 33.41\%, 69.30\%) & 3.31 \\
$I = 15$  & (1000, -29463, -166) & (23.08\%, 33.28\%, 22.79\%) & 4.17 \\
$I = 20$  & (1221, -28061, -146)  & (6.08\%, 26.94\%, 32.09\%) & 5.36 \\
$I = 25$ & (1234, -17366, -171) & (5.08\%, 21.44\%, 20.47\%) & 6.05\\
$I = 30$ & (1200, -19686, -200)  & (7.69\%, 10.95\%, 6.98\%) & 7.09\\
\bottomrule
\end{tabular}
\end{center}
\end{table}

\begin{table}[H]
\footnotesize
\caption{Multi-start for MPTGD}
\label{Computational Results}
\begin{center}
\renewcommand\arraystretch{0.8}
\begin{tabular}{p{2cm}p{3.5cm}p{3.5cm}p{2cm}}
\toprule
& $(s,l,t)$ & Relative Error & Time (Second)\\
\midrule
$I = 5$  & (1394, -20913, -214) & (7.23\%, 5.40\%, 0.47\%) & 5.19 \\
$I = 10$  & (1394, -20916, -214) & (7.23\%, 5.38\%, 0.47\%) & 8.79 \\
$I = 15$  & (1394, -20925, -214) & (7.23\%, 5.34\%, 0.47\%) & 12.11\\
$I = 20$ & (1394, -20927, -213)  & (7.23\%, 5.33\%, 0.93\%) & 16.66\\
$I = 25$ & (1393, -20930, -213) & (7.15\%, 5.32\%, 0.93\%) & 19.49\\
$I = 30$ & (1392, -21026, -212) & (7.08\%, 4.87\%, 1.40\%)  & 22.35\\
\bottomrule
\end{tabular}
\end{center}
\end{table}

Figure 6 shows that the cumulative regrets of MTGD and MPTGD decrease as the number of multi-starts $I$ increases. For MPTGD, when $I = 5,10,15,20,25$, there is a noticeable increase in the cumulative regret, which shows that it is necessary to choose a large $I$ to ensure that there is a high probability of convergence to the global optimum. 

\begin{figure}[!h]
\subfloat[]{
\begin{minipage}[c][1\width]{
0.5\textwidth}
\centering
\includegraphics[scale=.45]{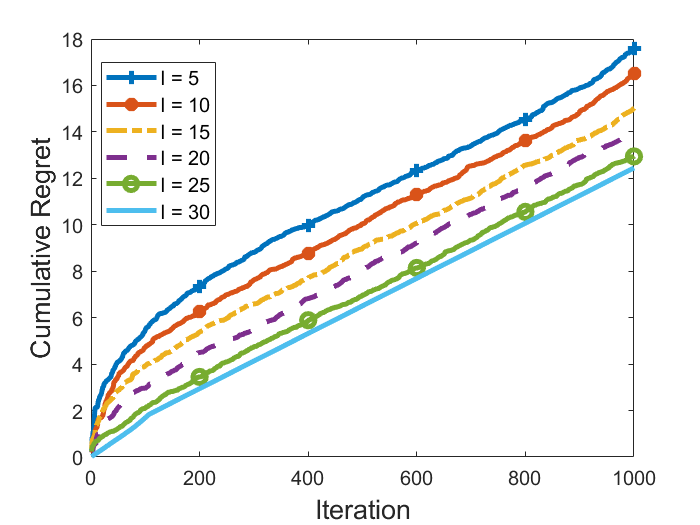}
\end{minipage}}
\subfloat[]{
\begin{minipage}[c][1\width]{
0.5\textwidth}
\centering
\includegraphics[scale=.45]{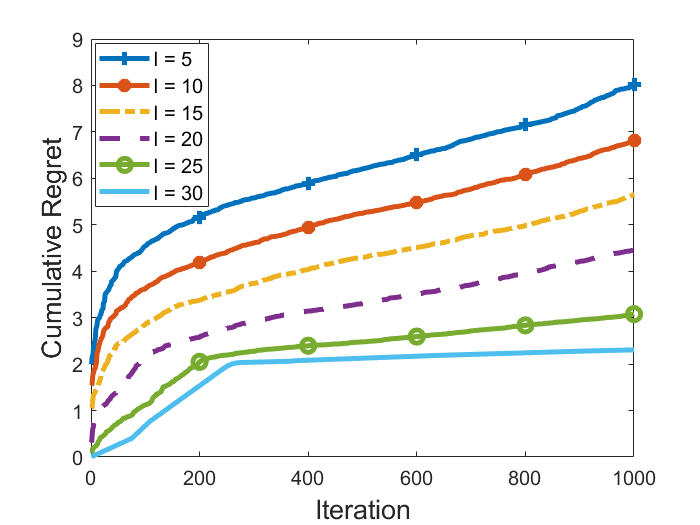}
\end{minipage}}
\caption{Empirical analysis of the cumulative regret for multi-start algorithms: (a) MTGD and (b) MPTGD}
\end{figure}
}}
\section{Conclusion}
In this paper, novel online non-convex learning algorithms for real-time river pollution source identification problem were developed and analyzed. {\color{black}{The identification problem studied in this paper had an instantaneously released pollution source with an initial value on source information identification.}} Our basic algorithm has vectorized and adaptive step sizes such that it ensured high estimation accuracy in dimensions with different magnitudes (released mass, location and time). In addition, the ``escaping from saddle points'' module was implemented to form the perturbed algorithm to further improve the estimation accuracy. Our basic algorithm has the local regret $O(N)$, and the perturbed algorithm has the local regret $O(N)$ with a high probability. The high probability cumulative regret bound $O(N)$ under particular condition on loss functions is also shown.

The experiments with artificially generated data validated all our theoretical regret bounds and showed that our algorithms were superior to existing online algorithms in all dimensions (e.g., released mass ($s$), location ($l$), and time ($t$)). Our algorithms were also implemented in Rhodamine WT dye concentration data from a travel time study on the Truckee River between Glenshire Drive near Truckee, California and Mogul, Nevada \cite{crompton_traveltime_2008}. We showed that the multi-start module and ``escaping from saddle point'' module could achieve a significantly low estimation error in certain dimensions (from 3.69\% to 1.21\% and from 14.42\% to 1.40\%). Thus, variants of online algorithms were provided for the decision maker to choose and implement, based on their own estimation accuracy requirement. 

For future research, observing that estimating the
gradients and project gradients would be computationally demanding
especially if the ADE functions are complicated. One possible
solution is to develop a ``bandit'' algorithm. Instead of observing the
full information of losses and computing their gradients on the entire
domain, only the function value queried by the decision
made would be observed, and the limited information would be used to construct an effective estimation
of the gradient.
\begin{acknowledgement*}
This research was supported by the National Research
Foundation (NRF), Prime Minister\textquoteright's Office, Singapore,
under its Campus for Research Excellence and Technological Enterprise
(CREATE) program. Wenjie Huang's research was also supported by the National Natural Science Foundation of China (Grants 72150002), and “HKU-100 Scholars” Research Start-up Funds. The authors are grateful to Professor William B. Haskell and Professor Yao Chen for their valuable comments and suggestions during the preparation of this paper. 
\end{acknowledgement*}
\bibliographystyle{apa}
\bibliography{SourceIdentification}

\begin{appendices}
\section*{Appendix}
The supplements of the paper are provided in this section.
\section{Multi-start Algorithms}

In this section, the multi-start version algorithms (Algorithm 7 and Algorithm 8) are developed
based on Algorithm 1 and 5, respectively.

\begin{algorithm}[H]
\caption{MTGD}

\textbf{Input}: sensor $m\in M$, window size $w\geq1$, tolerance
$\delta>0$, constant $W,\,K>0$, and a convex set $\mathcal{F}$ and group size $I$;

\textbf{Set} $x^{1}_{m},\,x_{m}^{1, (i)}\in\mathcal{F}$ for $i=1,\,...,\,I$
arbitrarily 

\textbf{for} $n=1,\,...,\,N$ \textbf{do}

$\quad$ Observe the cost function $\Psi^{n}_{m}:\,\mathcal{F}\rightarrow\mathbb{R}$.

$\quad$ Compute the initial step size: $\eta^{0} = \mathfrak{S}(F^{n}_{m,w},\,K)$. 

$\quad$\textbf{ for} $i=1,\,....,\,I$ \textbf{do}

$\quad$$\quad$Initialize $x_{m}^{n+1, (i)}:=x_{m}^{n, (i)}$.

$\quad$$\quad$Determine $\eta^{n,(i)}$ using Normalized Backtracking-Armijo
line search (Algorithm 2) for $x^{n+1, (i)}_{m}$.

$\quad$$\quad$$\quad$\textbf{while} $\|\nabla_{\mathcal{F},\eta^{n,(i)}}F^{n}_{m,\,w}(x_{m}^{n+1, (i)})\|>\delta$
\textbf{do}

$\quad\quad$$\quad$$\quad$Update $x_{m}^{n+1,(i)}=x_{m}^{n+1,(i)}-\eta^{n,(i)}\otimes\nabla_{\mathcal{F},\eta^{n,(i)}}F^{n}_{m,\,w}(x_{m}^{n+1, (i)})$.

$\quad$$\quad$$\quad$\textbf{end while}

\textbf{$\quad$Return} $x^{n+1}_{m}\in\arg\min_{x\in\left\{x_{m}^{n+1,(i)}\right\} {}_{i=1,...,I}}\left[\frac{1}{|N_m|}\sum_{n\in N_m}(C(x|l_{m},\,t_{m}^{n})-c_{m}^{n}) \right]^2.$

\textbf{end for}
\end{algorithm}

\begin{algorithm}[H]
\caption{MPTGD}

\textbf{Input}: sensor $m\in M$, window size $w\geq1$, constant $c\leq1$, constant $W,\,K>0$, and a convex set $\mathcal{F}$ and group size
$I$;

\textbf{Input:} $\mathcal{\chi}=3\max\left\{ \log\left(\frac{d\kappa\triangle f}{c\delta^{2}\epsilon}\right),\,4\right\} $,
$r=\frac{\sqrt{c}}{\chi^{2}}\cdot\frac{\delta}{\kappa}$, $g_{\textrm{thres}}=\frac{\sqrt{c}}{\mathcal{\chi}^{2}}\cdot\delta$,
$f_{\textrm{thres}}=\frac{c}{\mathcal{\chi}^{3}}\sqrt{\frac{\delta^{3}}{\iota}}$,
$t_{\textrm{thres}}=\lceil\frac{\chi}{c^{2}}\cdot\frac{\kappa}{\sqrt{\iota\,\delta}}\rceil$,
$t_{\textrm{noise}}=-t_{\textrm{thres}}-1$

\textbf{Set} $x^{1}_{m},\,x_{m}^{1,(i)}\in\mathcal{F}$ for $i=1,\,...,\,I$
arbitrarily

\textbf{for} $n=1,\,...,\,N$ \textbf{do}

$\quad$Observe the cost function $\Psi^{n}_{m}:\,\mathcal{F}\rightarrow\mathbb{R}$.

$\quad$Compute the initial step size: $\eta^{0} = \mathfrak{S}(F^{n}_{m,w},\,K)$. 

$\quad$Initialize $t_{\textrm{noise}}$; 

$\quad$\textbf{for} $i=1,\,....,\,I$ \textbf{do}

$\quad$$\quad$Initialize $x_{m}^{n+1,(i)}:=x_{m}^{n,(i)}$.

$\quad\quad$\textbf{for} $t=0,\,1,\,...$ \textbf{do}

$\quad\quad\quad$Determine $\eta^{t}$ using Normalized Backtracking-Armijo
line search (Algorithm 2).

$\quad\quad\quad$Set $x_{m}^{n+1,(i),\:t}=x_{m}^{n,(i)}$.

$\quad\quad\quad$\textbf{if} $\|\nabla_{\mathcal{F},\eta^{t}}F^{n}_{m,\,w}(x_{m}^{n+1,(i),\,t})\|\leq\delta$
and $t-t_{\textrm{noise}}>t_{\textrm{thres}}$ \textbf{then}

$\quad\quad\quad\quad$$t_{\textrm{noise}}=t$, $\tilde{x}_{m}^{n+1,(i),\,t}=x_{m}^{n+1,(i),\,t}$,
$x_{m}^{n+1,(i),\,t}=\tilde{x}_{m}^{n+1,(i),\,t}+\omega$, where $\omega$
uniformly sampled from $\mathbb{B}_{0}(r)$. 

$\quad\quad\quad$\textbf{if} $t-t_{\textrm{noise}}=t_{\textrm{thres}}$
and $F^{n}_{m,\,w}(x_{m}^{n+1, (i),\,t})-F^{n}_{m,\,w}(x_{m}^{n+1, (i),\,t_{\textrm{noise}}})>-f_{n}^{\textrm{thres}}$\textbf{
then}

$\quad\quad\quad\quad$\textbf{Return} $x_{m}^{n+1,(i)}=x_{m}^{n+1,(i),\,t_{\textrm{noise}}}$

$\quad\quad\quad$Update $x_{m}^{n+1,(i),\,t+1}=x_{m}^{n+1,(i),\,t}-\eta_{t}\otimes\nabla_{\mathcal{F},\eta_{t}}F^{n}_{m,\,w}(x_{m}^{n+1,(i),\,t})$.

$\quad\quad$ \textbf{end for}

\textbf{$\quad$end for}

\textbf{$\quad$Return} $x^{n+1}_{m}\in\arg\min_{x\in\left\{x_{m}^{n+1,(i)}\right\} {}_{i=1,...,I}}\left[\frac{1}{|N_m|}\sum_{n\in N_m}(C(x|l_{m},\,t_{m}^{n})-c_{m}^{n}) \right]^2$

\textbf{end for}
\end{algorithm}

\section{Supplements for Section 2}
{\color{black}{
\noun{Proof of Proposition \ref{prop_C_lipschitz_continuous}:} This proposition is proved by first giving the following proposition.
\begin{prop}
\label{Proposition A1} The first and second derivatives of $\tilde{C}_{n,m}$
on the interior of $\mathcal{F}$ are bounded.
\end{prop}

\begin{proof}
First, prove the boundedness of first derivative of $\tilde{C}_{n,m}(s,l,t)$
on $(s,\,l,\,t)$. There exist
\[
\frac{\partial\tilde{C}_{n,m}(s,l,t)}{\partial s}=\frac{1}{A\sqrt{4\pi D\left(t_{m}^{n}-t\right)}}\exp\left[-\frac{\left(l_{m}-l-v\left(t_{m}^{n}-t\right)\right)^{2}}{4D\left(t_{m}^{n}-t\right)}-k\left(t_{m}^{n}-t\right)\right],
\]
\[
\frac{\partial\tilde{C}_{n,m}(s,l,t)}{\partial l}=\frac{s}{A\sqrt{4\pi D\left(t_{m}^{n}-t\right)}}\cdot\frac{2\left(l_{m}-l-v\left(t_{m}^{n}-t\right)\right)}{4D\left(t_{m}^{n}-t\right)}\exp\left[-\frac{\left(l_{m}-l-v\left(t_{m}^{n}-t\right)\right)^{2}}{4D\left(t_{m}^{n}-t\right)}-k\left(t_{m}^{n}-t\right)\right],
\]
\begin{align*}
& \frac{\partial\tilde{C}_{n,m}(s,l,t)}{\partial t}=\frac{4\pi Ds}{2A[4\pi D\left(t_{m}^{n}-t\right)]^{3/2}}\exp\left[-\frac{\left(l_{m}-l-v\left(t_{m}^{n}-t\right)\right)^{2}}{4D\left(t_{m}^{n}-t\right)}-k\left(t_{m}^{n}-t\right)\right]\\
& \quad+\frac{s}{A\sqrt{4\pi D\left(t_{m}^{n}-t\right)}}\cdot\left[-\frac{(l_{m}-l)^{2}}{4D(t_{m}^{n}-t)^{2}}+\frac{v^{2}}{4D}+k\right]\exp\left[-\frac{\left(l_{m}-l-v\left(t_{m}^{n}-t\right)\right)^{2}}{4D\left(t_{m}^{n}-t\right)}-k\left(t_{m}^{n}-t\right)\right].
\end{align*}
Based on Assumption \ref{Assumption_1} (ii), for any $t\in\mathcal{T},$
set $t_{m}^{n}-t\neq0$. It is then clear that the first derivative
of $\tilde{C}_{n,m}(s,l,t)$ is continuous. Recall that the feasible
set $\mathcal{F}$ is a convex and compact set. By the boundedness of
compact sets in a metric space, it can be concluded that the first
derivative of $\tilde{C}_{n,m}(s,l,t)$ on $(s,\,l,\,t)$ is bounded. Following the same idea, the boundedness of the second derivatives
of $\tilde{C}_{n,m}$ on the interior of $\mathcal{F}$ can be proved. 
\end{proof}

To prove Proposition \ref{prop_C_lipschitz_continuous} (i), it is known from Proposition \ref{Proposition A1}, that the first derivative of $\tilde{C}_{n,m}(x)$
on $\mathcal{F}$ is bounded. From \cite{sohrab_topology_2014}, it can be shown that
$\tilde{C}_{n,m}(x)$ is Lipschitz continuous on $\mathcal{F}$ with
a certain modulus. As $\tilde{C}_{n,m}(x)$ is twice
differentiable on $\mathcal{F}$, then the boundedness
of the second derivatives from Proposition \ref{Proposition A1} leads to the Lipschitz continuity of $\nabla_{x}\tilde{C}_{n,m}$ on
$\mathcal{F}$.}}\\
\\
\noun{Proof of Proposition \ref{Proposition 2.5}:} Recall that $\tilde{C}_{n,m}(s,l,t)$
is Lipschitz continuous on the set $\mathcal{F}$ based on Proposition
\ref{prop_C_lipschitz_continuous}. Thus, given any $x,\,x^{\prime}\in\mathcal{F}$, it can be shown that $\left\Vert \tilde{C}_{n,m}(x)-\tilde{C}_{n,m}(x^{\prime})\right\Vert \leq\sigma\|x-x^{\prime}\|$.
As the square function on $\mathcal{F}$ is also Lipschitz continuous, it can be concluded that there
exists
\[
\left\Vert (\tilde{C}_{n,m}(x))^{2}-(\tilde{C}_{n,m}(x^{\prime}))^{2}\right\Vert \leq\kappa_{m}\|\tilde{C}_{n,m}(x)-\tilde{C}_{n,m}(x^{\prime})\|,\forall m\in M.
\]
Then, based on the triangle inequality, it can be shown that
\begin{align*}
\left\Vert \Psi_{m}^{n}(x)-\Psi_{m}^{n}(x^{\prime})\right\Vert  & =\left\Vert (\tilde{C}_{n,m}(x))^{2}-(\tilde{C}_{n,m}(x^{\prime}))^{2}\right\Vert \\
& \leq \kappa_{m}\|\tilde{C}_{n,m}(x)-\tilde{C}_{n,m}(x^{\prime})\|\\
& \leq \kappa_{m}\,\sigma\|x-x^{\prime}\|.
\end{align*}
Thus, there exists $\kappa=\kappa_{m}\,\sigma,\,\forall m\in M$ such
that $\left\Vert \Psi_{n}(x)-\Psi_{n}(x^{\prime})\right\Vert \leq\kappa\|x-x^{\prime}\|.$
For (ii) Recall that $\nabla_{x}f$ with $f(x) = x^{2}$ is Lipschitz continuous. For any
$x,\,x^{\prime}\in\mathcal{F}$,
\[
\|\nabla_{x}(\tilde{C}_{n,m}(x))^{2}-\nabla_{x}(\tilde{C}_{n,m}(x^{\prime}))^{2}\|\leq\tau_{m}\|x-x^{\prime}\|.
\]
Based on the triangle inequality,
\[
\|\nabla_{x}(\tilde{C}_{n,m}(x))^{2}\|\leq\|\nabla_{x}(\tilde{C}_{n,m}(x^{\prime}))^{2}\|+\tau_{m}\|x-x^{\prime}\|.
\]
Fix the value $x^{\prime}\in\mathcal{F}$. As the feasible set of
$\mathcal{F}$ is a convex and compact set, $\|\nabla_{x}(\tilde{C}_{n,m}(x))\|$
is bounded; that is, there exists $K_{1m}>0$ such that $\|\nabla_{x}(\tilde{C}_{n,m}(x))\|\leq K_{1m}$.
Besides, based on Proposition \ref{prop_C_lipschitz_continuous},
it can be shown that there exists $K_{2m}>0$ such that $\|\nabla_{x}\tilde{C}_{n,m}(x^{\prime})\|\leq K_{2m}$, $\left\Vert \nabla_{x}\tilde{C}_{n,m}(x)-\nabla_{x}\tilde{C}_{n,m}(x^{\prime})\right\Vert \leq\gamma\|x-x^{\prime}\|$ and 
\begin{align*}
& \left\Vert \nabla_{x}\Psi_{m}^{n}(x)-\nabla_{x}\Psi_{m}^{n}(x^{\prime})\right\Vert \\
= & \left\Vert \nabla_{x}(\tilde{C}_{n,m}(x))^{2}-\nabla_{x}(\tilde{C}_{n,m}(x^{\prime}))^{2}\right\Vert \\
= & \left\Vert \nabla_{x}(\tilde{C}_{n,m}(x))^{2}\nabla_{x}\tilde{C}_{n,m}(x)-\nabla_{x}(\tilde{C}_{n,m}(x^{\prime}))^{2}\nabla_{x}\tilde{C}_{n,m}(x^{\prime})\right\Vert \\
\leq & \left\Vert \nabla_{x}f(\tilde{C}_{n,m}(x))\nabla_{x}\tilde{C}_{n,m}(x)-\nabla_{x}f(\tilde{C}_{n,m}(x^{\prime}))\nabla_{x}\tilde{C}_{n,m}(x^{\prime})\right\Vert \\
= & \|\nabla_{x}(\tilde{C}_{n,m}(x))^{2}\nabla_{x}\tilde{C}_{n,m}(x)-\nabla_{x}(\tilde{C}_{n,m}(x))^{2}\nabla_{x}\tilde{C}_{n,m}(x^{\prime})\\
& +\nabla_{x}(\tilde{C}_{n,m}(x))^{2}\nabla_{x}\tilde{C}_{n,m}(x^{\prime})-\nabla_{x}(\tilde{C}_{n,m}(x^{\prime}))^{2}\nabla_{x}\tilde{C}_{n,m}(x^{\prime})\|\\
\leq & \left\Vert \nabla_{x}(\tilde{C}_{n,m}(x))^{2}\right\Vert \left\Vert \nabla_{x}\tilde{C}_{n,m}(x)-\nabla_{x}\tilde{C}_{n,m}(x^{\prime})\right\Vert \\
& +\left\Vert \nabla_{x}\tilde{C}_{n,m}(x^{\prime})\right\Vert \left\Vert \nabla_{x}(\tilde{C}_{n,m}(x))^{2}-\nabla_{x}(\tilde{C}_{n,m}(x^{\prime}))^{2}\right\Vert \\
\leq & (K_{1m}\,\gamma+K_{2m}\tau_{m})\|x-x^{\prime}\|.
\end{align*}
Let $\beta=\sum_{m\in M}(K_{1m}\,\gamma+K_{2m}\tau_{m})$.
Thus, it can be concluded that there exists $\beta$ such that
\[
\left\Vert \nabla_{x}\Psi_{m}^{n}(x)-\nabla_{x}\Psi_{m}^{n}(x^{\prime})\right\Vert \leq\beta\left\Vert x-x^{\prime}\right\Vert .
\]
For (iii), given the results in (ii) and implementing the same arguments
in (ii), the desired
results are derived.
\\
\\
{\color{black}{\noun{Proof of Proposition \ref{Proposition 2.7}:}} The proof follows from the proof of \cite[Proposition 2.3 ]{hazan_efficient_2017}. As $\nabla_{x}\Psi_{m}^{n}(x)$ is a continuous function on $\mathcal{F}$,
then the composition function $g(x):=x-\eta\otimes\nabla_{\mathcal{F},\eta}\Psi_{m}^{n}(x) = \Pi_{\mathcal{F}}\left[x-\eta\otimes\nabla_{x}\Psi_{m}^{n}(x) \right]$
is therefore continuous. Thus, $g$ satisfies the conditions for Brouwer's
fixed-point theorem (see \cite{brattka2016brouwer}), implying that there exists some $x^{\ast}\in\mathcal{F}$
for which $g(x^{\ast})=x^{\ast}$. At this point, the projected gradient
vanishes. 
}
\\
\\
\noun{Proof of Proposition \ref{Proposition 2.6}:} See \cite[Proposition 2.4]{hazan_efficient_2017}.
Let $u:=x+\eta\otimes\nabla\Psi(x)$, and $v:=u+\eta\otimes\nabla\Phi(x)$.
Define their respective projection $u^{\prime}=\Pi_{\mathcal{F}}[u]$,
and $v^{\prime}=\Pi_{\mathcal{F}}[v]$ so that $u^{\prime}=x-\eta\otimes\nabla_{\mathcal{F},\eta}\Psi(x)$,
and $v^{\prime}=x-\eta\otimes\nabla_{\mathcal{F},\,\eta}[\Psi+\Phi](x)$.
It can be first shown that $\|u^{\prime}-v^{\prime}\|\leq\|u-v\|$.

By the generalized Pythagorean theorem for convex sets, both
$\left\langle (u^{\prime}-v^{\prime})\oslash\eta,\,(v-v^{\prime})\oslash\eta\right\rangle \leq0$
and $\left\langle (v^{\prime}-u^{\prime})\oslash\eta,\,(u-u^{\prime})\oslash\eta\right\rangle \leq0$ are derived.
Combining these, it can be shown that 
\begin{align*}
& \left\langle (u^{\prime}-v^{\prime})\oslash\eta,\,[u^{\prime}-v^{\prime}-(u-v)]\oslash\eta\right\rangle \leq0\\
\Longrightarrow & \|(u^{\prime}-v^{\prime})\oslash\eta\|^{2}\leq\left\langle (u^{\prime}-v^{\prime})\oslash\eta,\,(u-v)\oslash\eta\right\rangle \\
& \leq\|(u^{\prime}-v^{\prime})\oslash\eta\|\cdot\|(u-v)\oslash\eta\|,
\end{align*}
as claimed. Finally, by the triangle inequality,
\begin{align*}
& \left\Vert \nabla_{\mathcal{F},\,\eta}[\Psi+\Phi](x)\right\Vert -\left\Vert \nabla_{\mathcal{F},\,\eta}\Psi(x)\right\Vert \\
\leq & \left\Vert \nabla_{\mathcal{F},\,\eta}[\Psi+\Phi](x)-\nabla_{\mathcal{F},\,\eta}\Psi(x)\right\Vert \\
= & \|(u^{\prime}-v^{\prime})\oslash\eta\|\\
\leq & \|(u-v)\oslash\eta\|\\
= & \left\Vert \nabla\Phi(x)\right\Vert ,
\end{align*}
as required.

\section{Supplements for Section 3}
\noun{Proof of Theorem \ref{Theorem 3.1}:} The theorem holds that the stopping criterion of Algorithm 2 is relaxed from the standard Backtracking-Armijo line search. Explicitly, the termination can also be proved by the limit of the Taylor series. Given the Taylor expansions with respect to the step size $\eta^{(l)}$ for the two terms,
\begin{align*}
&F^{n}_{m,\,w}(x^{n+1}_{m}-\eta^{(l)}\otimes\nabla_{\mathcal{F},\eta^{(l)}}F^{n}_{m,\,w}(x^{n+1}_{m})/\|\nabla_{\mathcal{F},\eta^{(l)}}F^{n}_{m,\,w}(x^{n+1}_{m})\|_{2})
\\
\leq &F^{n}_{m,\,w}(x^{n+1}_{m}) - \|\eta^{(l)}\|_{\min} \|\nabla_{\mathcal{F},\eta^{(l)}}F^{n}_{m,\,w}(x^{n+1}_{m})\|_{2}+ O(\|\eta ^{(l)}\|^{2}_{\infty}),
\end{align*}
and
\[F^{n}_{m,\,w}(x^{n+1}_{m})+\beta\|\eta^{(l)}\otimes\nabla_{\mathcal{F},\eta^{(l)}}F^{n}_{m,\,w}(x^{n+1}_{m})\|_{2}
\geq F^{n}_{m,\,w}(x^{n+1}_{m})+\beta~\|\eta^{(l)}\|_{\min}\|\nabla_{\mathcal{F},\eta^{(l)}}F^{n}_{m,\,w}(x^{n+1}_{m})\|_{2}.\]
Obviously, by examining the limit of the two terms when each element of $\eta^{(l)}$ is approaching zero, the condition in Algorithm 2 for entering the ``while'' loop will be violated, and thus the algorithm will terminate. 
\\
\\
\noun{Proof of Theorem \ref{Theorem 4.1}:} Prove Theorem \ref{Theorem 4.1} (i) by the following statement.
Given that $\eta^{n}$ is computed from Algorithm 2, it can be verified that
for all $x\in\mathcal{F}$, and $n\geq2$, so 
\begin{equation}
\|\nabla_{\mathcal{F},\eta^{0}}F_{n-1,\,w}(x)\|_{2}^{2}\leq\|\nabla_{\mathcal{F},\eta^{n}}F_{n-1,\,w}(x)\|_{2}^{2}.\label{fact}
\end{equation}
The above equality holds based on the definition of the projected gradient
as well as the fact that $\eta^{n}\leq\eta^{0}$. Note that Algorithm 1 only implements an iteration $x^{n}_{m}$ if $\|\nabla_{\mathcal{F},\eta^{n}}F^{n-1}_{m,\,w}(x^{n}_{m})\|_{2}^{2}\leq\delta$.
(Note that if $n=1$, $F^{n-1}_{m,\,w}$ is zero.) Let $h^{n}_{m}(x)=\frac{1}{w}\left[\Psi_{m}^{n}(x)-\Psi^{n-w}_{m}(x)\right]$,
which is $\frac{2\kappa}{w}$-Lipschitz. Then, for each $1\leq n\leq N$,
\begin{align*}
\|\nabla_{\mathcal{F},\eta^{0}}F^{n}_{m,\,w}(x^{n}_{m})\|_{2}^{2}= & \|\nabla_{\mathcal{F},\eta^{0}}\left[F^{n-1}_{m,\,w}+h_{m}^{n}\right](x^{n}_{m})\|_{2}^{2}\\
\leq & \|\nabla_{\mathcal{F},\eta^{0}}F^{n-1}_{m,\,w}(x^{n}_{m})\|_{2}^{2}+\|\nabla h^{n}_{m}(x^{n}_{m})\|_{2}^{2}\\
\leq & \|\nabla_{\mathcal{F},\eta^{n}}F^{n-1}_{m,\,w}(x^{n}_{m})\|_{2}^{2}+\|\nabla h^{n}_{m}(x^{n}_{m})\|_{2}^{2}\\
\leq & (\delta+\frac{2\kappa}{w})^{2}.
\end{align*}
for any $x\in\mathcal{F}$. The last equality holds based on equality
(\ref{fact}), Proposition \ref{Proposition 2.6}, and the stopping criterion of ATGD (Algorithm 1).

To prove Theorem \ref{Theorem 4.1} (ii), the following
lemmas are introduced.
\begin{lem}
\label{Lemma C1 } Let $\mathcal{F}$ be a closed convex set, and
let $\eta\in\mathbb{R}_{+}^{d}$. Suppose $\Psi:\,\mathcal{F}\rightarrow\mathbb{R}$
is differentiable. Then, for any $x\in\mathbb{R},$
\[
\left\langle \nabla\Psi(x),\,\eta^{2}\otimes\nabla_{\mathcal{F},\eta}\Psi(x)\right\rangle \geq\left\Vert \eta\otimes\nabla_{\mathcal{F},\eta}\Psi(x)\right\Vert _{2}^{2}.
\]
\end{lem}

\begin{proof}
Let $u=x-\eta\otimes\nabla\Psi(x)$, and $u^{\prime}=\Pi_{\mathcal{F}}\left[u\right]$.
Then
\begin{align*}
& \left\langle -\eta\otimes\nabla\Psi(x),\,-\eta\otimes\nabla_{\mathcal{F},\eta}\Psi(x)\right\rangle -\left\Vert \eta\otimes\nabla_{\mathcal{F},\eta}\Psi(x)\right\Vert _{2}^{2}\\
= & \left\langle u-x,\,u^{\prime}-x\right\rangle -\left\langle u^{\prime}-x,\,u^{\prime}-x\right\rangle \\
= & \left\langle u-u^{\prime},\,u^{\prime}-x\right\rangle \geq0,
\end{align*}
where the last inequality follows the generalized Pythagorean theorem. 
\end{proof}
For $2\leq n\leq N$, let $\tau^{n}$ be the number of gradient steps
taken in the outer loop at iteration $n-1$, in order to compute the
iteration $x^{n}_{m}$. For convenience, define $\tau^{1}=0$. A progress lemma during each gradient descent epoch is established: 
\begin{lem}
\label{Lemma 4.3} For any $2\leq n\leq N$, there exists $\eta^{\prime}\leq\eta^{0}$
such that 
\[
F^{n-1}_{m,\,w}(x^{n}_{m})-F^{n-1}_{m,\,w}(x^{n-1}_{m})\leq-\tau^{n}\left(\|\eta^{\prime}\|_{\min}-\frac{\beta\,\|\eta^{0}\|_{\infty}^{2}}{2}\right)\text{\ensuremath{\delta^{2}}}.
\]
\end{lem}

\begin{proof}
Consider a single iteration $z$ of the loop of $n$, and the
next iteration $z^{\prime}:=z-\eta^{n}\otimes\nabla_{\mathcal{F},\eta^{n}}F^{n-1}_{m,\,w}(z)$ occurs
when the step size is $\eta^{n}$. This can be shown by the $\beta$-smoothness
of $F^{n-1}_{m,\,w}$:
\begin{align*}
F^{n-1}_{m,\,w}(z^{\prime})-F^{n-1}_{m,\,w}(z) & \leq\left\langle \nabla F^{n-1}_{m,\,w}(z),\,z^{\prime}-z\right\rangle +\frac{\beta}{2}\|z^{\prime}-z\|_{2}^{2}\\
& =-\left\langle \nabla F^{n-1}_{m,\,w}(z),\,\eta^{n}\otimes\nabla_{\mathcal{F},\eta^{n}}F^{n-1}_{m,\,w}(z)\right\rangle +\frac{\beta}{2}\|\eta^{n}\otimes\nabla_{\mathcal{F},\eta^{n}}F^{n-1}_{m,\,w}(z)\|_{2}^{2}\\
& \leq-\|\eta^{n}\|_{\min}\left\Vert \nabla_{\mathcal{F},\eta^{n}}F^{n-1}_{m,\,w}(z)\right\Vert _{2}^{2}+\frac{\beta}{2}\|\eta^{n}\otimes\nabla_{\mathcal{F},\eta^{n}}F^{n-1}_{m,\,w}(z)\|_{2}^{2}\\
& \leq-\left(\|\eta^{\prime}\|_{\min}-\frac{\beta\,\|\eta^{0}\|_{\infty}^{2}}{2}\right)\left\Vert \nabla_{\mathcal{F},\eta^{n}}F^{n-1}_{m,\,w}(z)\right\Vert _{2}^{2}.
\end{align*}
The last equality holds when choosing $\eta^{\prime}\leq\eta^{n}$.
The algorithm only takes projected gradient steps when $\|\nabla_{\mathcal{F},\eta^{n}}F^{n-1}_{m,\,w}(z)\|\geq\delta$. {\color{black}{Then statement of this lemma can be derived given that
$F^{n-1}_{m,\,w}(x^{n}_{m})-F^{n-1}_{m,\,w}(x^{n-1}_{m}) = \tau^{n} \left(F^{n-1}_{m,\,w}(z^{\prime})-F^{n-1}_{m,\,w}(z)\right).$
}}
\end{proof}

To complete the proof of the theorem, write the telescopic sum:
\begin{align*}
F^{N}_{m,\,w}(x^{N}_{m}) & =\sum_{n=1}^{N}\left[F^{N}_{m,\,w}(x^{n}_{m})-F^{N-1}_{m,\,w}(x_{m}^{n-1})\right]\\
& =\sum_{n=1}^{N}\left[F^{n-1}_{m,\,w}(x^{n}_{m})-F^{n-1}_{m,\,w}(x^{n-1}_{m})+\Psi^{n}_{m,\,w}(x^{n}_{m})-\Psi^{n-w}_{m,\,w}(x^{n}_{m})\right]\\
& \leq\sum_{n=2}^{N}\left[F^{n-1}_{m,\,w}(x_{m}^{n})-F^{n-1}_{m,\,w}(x_{m}^{n-1})\right]+\frac{2B\,N}{w}.
\end{align*}
Using Lemma \ref{Lemma 4.3},
\[
F^{N}_{n,\,w}(x^{N}_{m})\leq\frac{2B\,N}{w}-\left(\|\eta^{\prime}\|_{\min}-\frac{\beta\,\|\eta^{0}\|_{\infty}^{2}}{2}\right)\delta^{2}\cdot\sum_{n=1}^{N}\tau^{n},
\]
hence,
\begin{align*}
\sum_{n=1}^{N}\tau^{n} & \leq\frac{1}{\delta^{2}\left(\|\eta^{\prime}\|_{\min}-\frac{\beta\,\|\eta^{0}\|_{\infty}^{2}}{2}\right)}\cdot\left(\frac{2B\,N}{w}-F^{N}_{m,\,w}(x^{N}_{m})\right)\\
& \leq\frac{2B}{\delta^{2}\left(\|\eta^{\prime}\|_{\min}-\frac{\beta\,\|\eta^{0}\|_{\infty}^{2}}{2}\right)w}N,
\end{align*}
as claimed.\\

\section{Supplements for Section 4}
{\color{black}{\noun{Proof of Proposition \ref{Proposition 4.7}}: The proposition states that the $\mathcal{F}$ space can be divided into three regions: 1) a region
where the gradient is large; 2) a region where the Hessian has a significant
negative eigenvalue (around saddle point); and 3) a region close
to a local minimum. As $\zeta>0$ can be arbitrarily chosen, the condition in Proposition \ref{Proposition 4.7} (iii) holds for all $x\in\mathcal{F}$.}}\\
\\
\noun{Proof of Theorem \ref{Theorem 4.8 }: }Recall
that an $\tilde{\epsilon}$-second-order stationary point has a small gradient, where the Hessian does not have a significant
negative eigenvalue. Suppose that currently, at an iteration $x^{n,t}_{m}$
that is not an $\tilde{\epsilon}$-second-order stationary point,
there are two possibilities:
\begin{enumerate}
\item Gradient is large: $\|\nabla_{x}F^{n}_{m,\,w}(x_{m}^{n,t})\|_{2}\geq g_{\textrm{thres}}$,
or
\item Around the saddle point: $\|\nabla_{x}F^{n}_{m,\,w}(x_{m}^{n,t})\|_{2}\leq g_{\textrm{thres}}$
and $\lambda_{\min}\left(\nabla_{x}^{2}F^{n}_{m,\,w}(x_{m}^{n,t})\right)\leq\sqrt{\iota~\tilde{\epsilon}}$.
\end{enumerate}
The following two lemmas address the above two possibilities, which guarantee that the perturbed gradient descent will
decrease the function value in both scenarios. The next lemma shows that if the current gradient is large, progress is made
in the function value in proportion to the square of the norm of the gradient.
\begin{lem}
Assume that $\Psi^{n}_{m}$ satisfies the condition in
Proposition \ref{Proposition 2.5}. Then, for the gradient descent with the step size $\|\eta\|_{\infty}\leq1/\kappa$, there exists
$\eta^{\prime}\leq\eta$, 
\[
F^{n}_{m,\,w}(x_{m}^{n,t+1})\leq F^{n}_{m,\,w}(x_{m}^{n,t})-\left(\|\eta^{\prime}\|_{\min}-\frac{\|\eta^{0}\|_{\infty}}{2}\right)\|\nabla_{x}F^{n}_{m,\,w}(x_{m}^{n,t})\|_{2}^{2}.
\]
\end{lem}

\begin{proof}
By the properties in Proposition \ref{Proposition 2.5},
\begin{align*}
F^{n}_{m,\,w}(x_{m}^{n,t+1})\leq & F^{n}_{m,\,w}(x_{m}^{n,t})+\nabla_{x}F^{n}_{m,\,w}(x_{m}^{n,t})^{\top}(x_{m}^{n,t+1}-x_{m}^{n,t})+\frac{\kappa}{2}\|x_{m}^{n,t+1}-x_{m}^{n,t}\|_{2}^{2}\\
\leq & F^{n}_{m,\,w}(x_{m}^{n,t})-\|\eta^{t}\|_{\min}\|\nabla_{x} F^{n}_{m,\,w}(x_{m}^{n,t})\|_{2}^{2}+\frac{\kappa\,\|\eta^{t}\|_{\infty}^{2}}{2}\|\nabla_{x} F^{n}_{m,\,w}(x_{m}^{n,t})\|_{2}^{2}\\
\leq & F^{n}_{m,\,w}(x_{m}^{n,t})-\left(\|\eta^{\prime}\|_{\min}-\frac{\|\eta^{0}\|_{\infty}}{2}\right)\|\nabla_{x} F^{n}_{m,\,w}(x_{m}^{n,t})\|_{2}^{2}
\end{align*}
\end{proof}
The next lemma shows that a perturbation followed by
a small number of standard gradient descent steps can also make the
function value decrease with a high probability.
\begin{lem}
\label{Lemma D2} There exists an absolute constant
$c_{\textrm{max}}$, for $\Psi_{m}^{n}$ satisfies properties in Proposition
\ref{Proposition 2.5} and \ref{Proposition 4.7}, and any $c\leq c_{\textrm{max}}$,
and $\chi\geq1$. Let $\eta_{0},\,r,\,g_{\textrm{thres}},\,f_{\textrm{thres}},\,t_{\textrm{thres}}$ be
computed as Algorithm 3. Then, if $\tilde{x}_{m}^{n,t}$ satisfies:
\[
\|\nabla_{x} F^{n}_{m,\,w}(\tilde{x}_{m}^{n,t})\|_{2}\leq g_{\textrm{thres}},\quad\textrm{and}\quad\lambda_{\min}\left(\nabla_{x}^{2} F^{n}_{m,\,w}(\tilde{x}_{m}^{n,t})\right)\geq-\sqrt{\iota\,\tilde{\epsilon}},
\]
let $x_{m}^{n,t}=\tilde{x}_{m}^{n,t}+\omega_{t}$, where $\omega_{t}$
comes from the uniform distribution over $\mathbb{B}_{0}(r)$, and
let $x_{m}^{n,t+1}$ be the iterations of the gradient descent from $x_{m}^{n,t}$
with the step size $\eta^{t}$, with at least the probability $1-\frac{d\,\kappa}{\sqrt{\iota\,\epsilon}}\exp(-\chi)$,
\[
F^{n}_{m,\,w} (x_{m}^{n,t+t_{\textrm{thres}}})-F^{n}_{m,\,w} (\tilde{x}_{m}^{n,t})\leq-f_{\textrm{thres}}.
\]
\end{lem}

In \cite[Section 5.2]{jin_how_2017} and \cite[Section A.2 ]{jin_how_2017},
detailed proofs have been derived for fixed step-size version of Lemma
\ref{Lemma D2}. The main idea is listed as follows. After adding
a perturbation, the current point of the algorithm comes from a uniform
distribution over a $d$-dimensional ball centered at $\tilde{x}_{m}^{n,t}$. This perturbation ball can be divided into two adjoined regions: (1)
an \emph{escaping} region, which
consists of all the points whose function value decrease by a least
$f_{\textrm{thres}}$ after $t_{\textrm{thres}}$ steps; (2) a \emph{stuck}
region. The proof is to show that the stuck region only consists of
a small proportion of the volume of the perturbation ball, and the
current point has a small chance of falling in the stuck region.
The proofs of Lemma \ref{Lemma D2} will also hold for vectorized
and adjusted step sizes $\eta^{t}$, as $\eta^{t}\leq\eta^{0}$
(see \cite[Lemma 14, 15, 16 and 17]{jin_how_2017}).

To prove Theorem \ref{Theorem 4.8 } following \cite[Appendix A]{jin_how_2017},
suppose in epoch $n$, the point $x^{0}$ will be the starting point that $\|\nabla_{x} F^{n}_{m,\,w}(x^{0})\|_{2}\leq g_{\textrm{thres}}$.
Then, Algorithm 5 will add perturbation and check the termination condition.
If the condition is not met, it must follow that
{\color{black}{\[
F^{n}_{m,\,w}(x_{m}^{{n,t_{\textrm{thres}}}})-F^{n}_{m,\,w}(x^{0})\leq-f_{\textrm{thres}}=-\frac{c}{\chi^{3}}\cdot\sqrt{\frac{\delta^{3}}{\iota}}.
\]}}
This means that on average, every step decreases the function value by
{\color{black}{\[
\frac{F^{n}_{m,\,w}(x_{m}^{n,{t_{\textrm{thres}}}})-F^{n}_{m,\,w}(x^{0})}{t_{\textrm{thres}}}\leq-\frac{c^{3}}{\chi^{4}}\cdot\frac{\delta^{2}}{\iota},
\]}}which further means that in each epoch $n$, Algorithm 5 will terminate
within the following number of iterations:
\[
\frac{\Delta f}{\frac{c^{3}}{\chi^{4}}\cdot\frac{\delta^{2}}{\iota}}=\frac{\chi^{4}}{c^{3}}\cdot\frac{\iota\,\Delta f}{\delta^{2}}=O\left(\frac{\iota\,\Delta f}{\delta^{2}}\log^{4}\left(\frac{d\,\iota\,\Delta f}{\delta^{2}\epsilon}\right)\right).
\]
It can also be shown that in each epoch, when Algorithm 5 terminates,
the point it finds is actually an $\delta$-second-order stationary
point of $F^{n}_{m,w}$. However, during the entire run of Algorithm
5 in each epoch, the number of times the perturbations are added is at most
\[
\frac{1}{t_{\textrm{thres}}}\cdot\frac{\chi^{4}}{c^{3}}\cdot\frac{\iota\,\Delta f}{\delta^{2}}=\frac{\chi^{3}}{c}\frac{\sqrt{\iota\,\delta}\Delta f}{\delta^{3}}.
\]
{\color{black}{Finally, the probability that Algorithm 5 terminates and finds a $\delta$-second-order stationary point is at least}}
\[
1-\frac{d\,\kappa}{\sqrt{\iota\,\epsilon}}\exp(-\chi)\frac{\chi^{3}}{c}\frac{\sqrt{\iota\,\delta}\Delta f}{\delta^{3}}\geq1-\delta,
\]
given the formulation of $\chi$.\\
\noun{}\\
\noun{Proof of Theorem \ref{Theorem 4.9 }: }With probability $1-\epsilon$,
and after \[O\left(\frac{\kappa\triangle f\,w^{2}}{\delta^{2}}\log^{4}\left(\frac{d\kappa\triangle fw^{2}}{\delta^{2}\epsilon}\right)\,N\right),\]
gradient estimations, the {\color{black}{long-run average cumulative regret}}, denoted
by $\mathfrak{R}^{G}_{m}$, is as follows until time $N$ is
\begin{align*}
\frac{\mathfrak{R}^{G}_{m}(N)}{N} = & \frac{1}{N}\left[\sum_{n=1}^{N}\Psi_{m}^{n}(x_{m}^{n})-\inf_{x\in\mathcal{F}}\sum_{n=1}^{N}\Psi_{m}^{n}(x)\right]\\
\leq & \frac{1}{N}\left[\sum_{n=1}^{N}\Psi_{m}^{n}(x_{m}^{n})-\sum_{n=1}^{N}\Psi_{m}^{n}(x_{m}^{n,\ast})\right]\\
\leq & \frac{\mu}{\kappa} \frac{1}{N} \sum_{n=1}^{N} \|\nabla_{x} \Psi^{n}_{m}(x_{m}^{n})\|\\
= & \frac{\mu\delta}{\kappa}.
\end{align*}
{\color{black}{Assumption \ref{Assumption 4.5} can be ensured with a high probability. The probability of attaining the cumulative regret bound is thus no higher than $1-\epsilon$.}}
\end{appendices}
\end{document}